\documentclass{article}
\PassOptionsToPackage{numbers, compress}{natbib}

 \usepackage{algorithm}
 \usepackage{algorithmic}
 \usepackage{amsmath}
 \usepackage{mathtools}
 \usepackage{graphicx}
 \usepackage{amsthm}
 
 \usepackage[table]{xcolor}
 \usepackage[numbers, compress]{natbib}
 
 \newtheorem{proposition}{Proposition}

 \newtheorem{remark}{Remark}

 \usepackage{xcolor}
 \usepackage{wrapfig}
 \usepackage{subcaption}

 \usepackage[preprint]{neurips_2026}

\usepackage[utf8]{inputenc} 
\usepackage[T1]{fontenc}    
\usepackage{hyperref}       
\usepackage{url}            
\usepackage{booktabs}       
\usepackage{amsfonts}       
\usepackage{nicefrac}       
\usepackage{microtype}      
\usepackage{xcolor}         

\title{VASR: Variance-Aware Systematic Resampling for Reward-Guided Diffusion}

\author{%
Shivanshu Shekhar$^{1}$ \quad
Sagnik Mukherjee$^{1}$ \quad
Jia Yi Zhang$^{2}$ \quad
Tong Zhang$^{1}$ \\
$^{1}$ Siebel School of Computing and Data Science \quad
$^{2}$ Department of Statistics \\
University of Illinois Urbana-Champaign \\
{\tt\small \{shekhar6, sagnikm3, jyzhang5, tozhang\}@illinois.edu}
}

\begin{document}

\maketitle

\begin{abstract}
Sequential Monte Carlo (SMC) samplers for reward-guided diffusion models often suffer from rapid lineage collapse: a few high-reward particles dominate the population within a handful of resampling steps, destroying diversity and degrading sample quality. We propose a variance-decomposition framework for reward-guided diffusion SMC that separates continuation variance $V_t^{\mathrm{cont}}$ from residual variance $V_t^{\mathrm{res}}$, revealing that high offspring-count variance under the commonly used multinomial resampling drives this collapse. This motivates \textsc{VASR} (Variance-Aware Systematic Resampling), which addresses both variance terms via variance-optimal mass allocation $m_t \propto w_t e^{r_t}$ (minimizing $V_t^{\mathrm{cont}}$) and systematic resampling (controlling $V_t^{\mathrm{res}}$). For latent diffusion models where intermediate rewards are noisy due to stochastic continuations, we propose \textsc{VASR-Max}, a deliberately biased high-selection variant for variance-sensitive reward optimization. Both methods are training-free, fully parallelizable, and add only linear overhead. On MNIST and CIFAR-10, \textsc{VASR} achieves as high as $26\%$ better FID than prior SMC methods while remaining $\sim\!66\times$ faster than MCTS-based value methods at matched compute. On text-to-image generation, \textsc{VASR-Max} consistently outperforms the strongest SMC baseline across compute budgets and matches MCTS-based methods within $2.5$--$3\%$ reward at high budgets while being approximately $4\times$ faster.
\end{abstract}

\section{Introduction}
\label{sec:intro}

Diffusion models~\citep{ddpm,smld} have become a dominant paradigm for generative modeling, achieving state-of-the-art performance across modalities including images~\cite{ddpm, sdxl, ld}, video~\cite{cogvideox, veo3}, and language~\cite{llada2, mdlm}. In practical settings, however, generation must satisfy objectives beyond reproducing the training distribution, such as generating images that aligns with human preferences. A common requirement is \emph{reward alignment}, where generated samples should remain on the learned data manifold while maximizing a reward function $r : \mathcal{X} \rightarrow \mathbb{R}$ that captures human preferences, task objectives, or domain-specific constraints. Existing approaches to reward alignment broadly fall into two categories:

\paragraph{RL-based fine-tuning}
These approaches formulate alignment as a reinforcement learning objective applied to a pretrained generative model:
\begin{equation}
\mathcal{L}_{\text{RLHF}} =
\mathbb{E}_{\tau \sim \pi_\theta} \left[ R(\tau) \right]
- \beta \, \mathrm{KL}(\pi_\theta \| \pi_{\theta_{\text{ref}}}),
\label{eq:rlhf}
\end{equation}
where $\pi_{\theta_{\text{ref}}}$ is a reference model and the KL term regularizes deviations from the pretrained distribution. Instantiations include policy gradient methods~\cite{ddpo, dpok}, direct preference optimization~\cite{d3po, diffdpo}, and direct reward optimization~\cite{draft}. While effective, these methods require costly fine-tuning and must be retrained whenever the reward function changes.

\paragraph{Inference-time alignment}
An alternative direction avoids retraining by modifying the sampling procedure of a frozen diffusion model to target $\pi^{*}(x) \propto p_{\theta}(x)\exp(\lambda r(x)),$ where $p_{\theta}$ is the pretrained model and $\lambda \ge 0$ controls alignment strength. These methods allow rapid adaptation to new reward functions without retraining or additional compute-heavy optimization, and hence are usually preferred over finetuning based methods.

Several families of inference-time methods have been proposed. Gradient-based approaches~\cite{universal_guidance, dps} bias the denoising trajectory using reward gradients, which requires differentiable reward models that may destabilize the denoising process by pulling trajectories off the data manifold \cite{dts}. Search-based methods~\cite{uhera, uhera2} perform local reward-guided exploration, while value based methods~\cite{dts} learn a value function via Monte Carlo rollouts and sample greedily, though both families are computationally impractical at scale. Particle-based methods sidestep both pitfalls: they scale tractably and do not require differentiable reward models. These methods~\cite{smc1, dou2024diffusion, fkd} maintain a population of K particles during inference and, at each resampling step, draw a new population by sampling particles in proportion to their current reward estimates. However, existing SMC methods suffer from diversity collapse.

We trace this phenomenon of diversity collapse to the commonly used multinomial resampling in current SMC based methods, which introduces offspring-count fluctuations of order $\mathcal{O}(K)$ per particle, leading to rapid and often irreversible loss of diversity that severely degrades sample quality and distributional coverage. We address this through a \textbf{variance-decomposition framework} that separates the total estimator variance into two components: continuation variance $V_t^{\mathrm{cont}}$ from stochastic particle propagation, and residual variance $V_t^{\mathrm{res}}$ from random offspring counts. We minimize $V_t^{\mathrm{cont}}$ through variance-optimal mass allocation stratergy (\S \ref{sec:method}) derived from the variance decomposition, and control $V_t^{\mathrm{res}}$ using systematic resampling (\S \ref{sec:prelim}) which bound offspring deviations as $|N_t^{(k)} - m_t^{(k)}| \le 1$ \cite{sys-2}. The resulting sampler, \textsc{VASR} (Variance-Aware Systematic Resampling), addresses both sources of variance in a principled manner. We also extend \textsc{VASR} to latent diffusion settings with \textsc{VASR-Max}, where intermediate Tweedie proxy rewards are highly noisy and the objective is to obtain a single high-quality sample; in this regime, concentrating offspring on the highest-scoring particle yields more robust selection.

We validate \textsc{VASR} across a range of three models, four tasks, and five scales and find that it consistently improves over existing approaches. In particular, it yields higher sample quality, better diversity, and more stable inference compared to prior SMC-based and value-based baselines. Our findings demonstrate that principled variance control is key to improving diffusion-based inference, establishing \textsc{VASR} as a scalable and effective approach for inference-time alignment. Our main contributions are:

\begin{enumerate}

\item \textbf{Variance decomposition for reward-guided diffusion SMC.}
We decompose per-step estimator variance into continuation variance $V_t^{\mathrm{cont}}$ (from particle propagation) and residual variance $V_t^{\mathrm{res}}$ (from offspring sampling). Both terms degrade sample quality under aggressive reward tilting, motivating paired control strategies. 


\item \textbf{Paired variance control via reweighting and bounded resampling.}
The continuation term $V_t^{\mathrm{cont}}$ is minimized through variance-optimal mass allocation $m_t^{(k)} \propto w_t^{(k)} \exp(\hat{r}_t^{(k)})$ derived from the decomposition. The residual term $V_t^{\mathrm{res}}$ is bounded via systematic resampling, which reduces per-particle variance from $\mathcal{O}(K)$ to $\mathcal{O}(1)$ \cite{sys-2}. This paired strategy yields the \textsc{VASR} sampler. 

\item \textbf{Empirical validation of the variance mechanism.}
Across MNIST, CIFAR-10, and Stable Diffusion, \textsc{VASR} improves distributional quality and lineage diversity over SMC baselines while remaining fully parallelizable. Ablations confirm the mechanism: jointly controlling both variance terms preserves lineages, improves FID, and yields better reward scaling under matched compute.
\end{enumerate}

\section{Related Work}
Existing inference-time alignment approaches for diffusion models fall into three broad categories: gradient-based guidance, search/value-based methods, and particle-based sampling.

\textbf{Gradient-based Guidance.}
Gradient-based methods steer the model towards the desired objective~\cite{universal_guidance, dps, he2023manifoldpreservingguideddiffusion,song2021scorebasedgenerativemodelingstochastic, song2023pseudoinverseguided} using the gradient of the value or score function. This is closely related to classifier guidance \cite{dhariwal2021diffusionmodelsbeatgans} and has been applied to incorporate semantic constraints, reward models, and preference signals during sampling. While effective, these methods are inapplicable to non-differentiable or discrete-state settings and introduces per-step gradient cost that scales with model size.



\textbf{Search and Value-based Methods.}
Search-based methods generate multiple candidate denoising transitions and select higher-reward trajectories, either via greedy local search or by framing inference as a search problem over noise trajectories with verifier feedback~\cite{uhera, uhera2, search1}. While these approaches can improve sample quality, they typically rely on shallow look-ahead and may overfit to verifier biases under aggressive search~\cite{search1}. Value function-based methods instead estimate soft values at intermediate states to guide sampling. Diffusion Tree Sampling (DTS)~\cite{dts} formulates the reverse process as a finite-horizon tree and applies Monte Carlo Tree Search with soft Bellman backups, achieving strong asymptotic guarantees and improved compute efficiency. However, its sequential tree construction limits parallelism: at matched NFE budgets, DTS is 4-66$\times$ slower than fully parallel particle-based methods~\cite{dts}.

\textbf{Particle-based Methods.}
Particle-based methods~\cite{tds, dou2024diffusion, fkd, trippe2023diffusionprobabilisticmodelingprotein} propagate multiple candidate trajectories (particles) simultaneously through the diffusion process, resampling them at intermediate steps according to potential functions that approximate the soft value. Sequential Monte Carlo (SMC) enables reward-guided generation with frozen model weights. SMC is theoretically guaranteed to recover the target distribution given exact potentials and infinite particles \cite{fk}. The core bottleneck for sample quality in particle-based diffusion alignment is that the repeated resampling step can dramatically reduce diversity, leading to particle degeneracy and lineage collapse when a few high-weight trajectories dominate the population.

Our approach is particle-based but distinguished by a variance-decomposition framework that separates continuation and residual variance. This decomposition reveals both terms as bottlenecks in reward-guided diffusion and motivates paired control: variance-optimal mass allocation for $V_t^{\mathrm{cont}}$ and systematic resampling for $V_t^{\mathrm{res}}$.

\section{Preliminaries}
\label{sec:prelim}

\paragraph{Diffusion Models.}
A diffusion model defines a forward noising process
$
q(x_t \mid x_0)
=
\mathcal{N}\!\left(
x_t;\,
\sqrt{\bar{\alpha}_t}\,x_0,\,
(1-\bar{\alpha}_t)I
\right),
$
where $\bar{\alpha}_t = \prod_{s=1}^{t}\alpha_s$. A neural network $\epsilon_\theta(x_t,t)$ is trained to predict the injected noise, implicitly learning the score of the data distribution. Given a noisy sample $x_t$, the Tweedie estimate of the clean sample is
\begin{equation}
\hat{x}_0(x_t,t)
=
\frac{x_t - \sqrt{1-\bar{\alpha}_t}\,
\epsilon_\theta(x_t,t)}{\sqrt{\bar{\alpha}_t}},
\end{equation}
which corresponds to the posterior mean $\mathbb{E}[x_0 \mid x_t]$~\cite{ddim}. Sampling is typically performed using DDIM~\cite{ddim}. We write $\mathrm{DDIM}_{\eta}(x_t,t)$ for the DDIM update from $x_t$ to $x_{t-1}$ with stochasticity parameter $\eta \in [0,1]$, given by
\begin{equation}
x_{t-1}
=
\sqrt{\bar{\alpha}_{t-1}}\,\hat{x}_0(x_t,t)
+
\sqrt{1-\bar{\alpha}_{t-1} - \sigma_t^2(\eta)}\,
\epsilon_\theta(x_t,t)
+
\sigma_t(\eta)\,\xi_t,
\qquad
\xi_t \sim \mathcal{N}(0, I),
\end{equation}
where $\sigma_t^2(\eta) = \eta^2 \cdot \tfrac{1 - \bar{\alpha}_{t-1}}{1 - \bar{\alpha}_t}\bigl(1 - \tfrac{\bar{\alpha}_t}{\bar{\alpha}_{t-1}}\bigr)$ controls the scheduler-scaled noise injection. Setting $\eta = 0$ recovers the deterministic update in which the only randomness arises from the initial noise $x_T \sim \mathcal{N}(0,I)$, while $\eta = 1$ recovers the DDPM \cite{ddpm} ancestral sampler.

\paragraph{Reward-Tilted Sampling via Particle Methods.}
Let $r : \mathcal{X} \rightarrow \mathbb{R}$ denote a reward function. Inference-time alignment aims to sample from the reward-tilted distribution $p^*(x_0) \propto p_\theta(x_0)\exp(\lambda r(x_0)),$ where $p_\theta$ is the pretrained diffusion model and $\lambda \ge 0$ controls alignment strength. Since $r(x_0)$ is inaccessible at intermediate noisy states $x_t$, it is usually approximated via the Tweedie estimate $\hat{x}_0(x_t, t)$, making direct optimization tractable. Particle methods approximate $p^*$ using $K$ weighted particles $\{(x_t^{(i)}, w_t^{(i)})\}_{i=1}^K$ that alternate between selection and mutation. At each resampling step $t \in \mathcal{T}$, weights are updated using a potential $G_t(x_t^{(i)})$, typically defined as a function of rewards computed via the Tweedie proxy: $w_t^{(i)} \;\propto\; w_{t+1}^{(i)}\, G_t(x_t^{(i)})$. Particles are then propagated via DDIM, and the empirical measure $\hat{\pi}_t^K = K^{-1}\sum_{i=1}^K \delta_{x_t^{(i)}}$ (where $\delta_{x}$ denotes the Dirac delta at $x$) converges to $p^*$ as $K \to \infty$~\citep{fk}. To mitigate weight degeneracy, resampling periodically replaces low-weight particles with copies of high-weight ones. However, poorly designed resampling destroys particle diversity, rendering the approximation uninformative.

\paragraph{Resampling: Multinomial vs.\ Systematic.} A resampling step replaces the current particle set by offspring counts $(N_t^{(1)},\ldots,N_t^{(K)})$ drawn with respect to positive masses $m_t^{(1)},\ldots,m_t^{(K)}$ with $\sum_k m_t^{(k)} = K$, where $m_t^{(k)}$ encodes how many offspring particle $k$ should receive on average. The method is \emph{unbiased} if $\mathbb{E}[N_t^{(k)}] = m_t^{(k)}$ and $\sum_k N_t^{(k)} = K$ almost surely, in which case assigning offspring weight $\tilde w_t^{(k)} = w_t^{(k)}/m_t^{(k)}$ preserves the weighted empirical measure in expectation~(Proposition~\ref{prop:unbiased}, Appendix~\ref{app:vasr-allocation}). \emph{Multinomial resampling} draws 
\begin{equation}
(N_t^{(1)},\ldots,N_t^{(K)}) \sim \mathrm{Multinomial}(K;\,m_t^{(1)}/K,\ldots,m_t^{(K)}/K),\
\label{eq:multinomial}
\end{equation}
 giving $\mathrm{Var}(N_t^{(k)}) = m_t^{(k)}(1 - m_t^{(k)}/K)$. Under aggressive reward tilting, a dominant particle with $m_t^{(k)} = \Theta(K)$ exhibits offspring-count variance $\mathcal{O}(K)$, directly inflating the variance of any downstream estimator~\citep{fkd,sys-ref, sys-2}. \emph{Systematic resampling}~\citep{sys-ref, sys3, sys-2} realizes the same masses using a single shared random shift, reducing this to $\mathcal{O}(1)$ per particle — a property we exploit in \textsc{VASR} and derive in Section~\ref{sec:method}.

\section{Variance-Aware Systematic Resampling}
\label{sec:method}

\begin{algorithm}[!htbp]
\caption{\textsc{VASR} / \textsc{VASR-Max}: Variance-Aware Systematic Resampling}
\label{alg:vasr}
\begin{algorithmic}[1]
\REQUIRE Frozen model $p_{\theta}$, reward $r$, particle count $K$, alignment strength $\lambda$, schedule $\mathcal{T} \subset \{1,\ldots,T\}$, mode $\in \{\textsc{VASR}, \textsc{VASR-Max}\}$
\STATE $x_T^{(i)} \sim \mathcal{N}(0,I)$, $w_T^{(i)} \leftarrow 1$ for $i=1,\ldots,K$
\FOR{$t = T, T-1, \ldots, 1$}
  \IF{$t \in \mathcal{T}$}
    \STATE $\tilde{r}_i \leftarrow r(\hat{x}_0(x_t^{(i)},t))$; optionally z-score: $\tilde{r}_i \leftarrow (\tilde{r}_i - \bar{r}_t)/(\sigma_t + \varepsilon)$
    \IF{mode $= \textsc{VASR}$}
      \STATE $m_t^{(i)} \leftarrow K \cdot \dfrac{w_t^{(i)} \exp(\tilde{r}_i)}{\sum_j w_t^{(j)} \exp(\tilde{r}_j)}$, \quad $\tilde w_t^{(i)} \leftarrow w_t^{(i)} / m_t^{(i)}$
      \STATE Reorder by descending $w_t^{(i)} \exp(\tilde{r}_i) / m_t^{(i)}$; draw $U \sim \mathrm{Uniform}(0,1)$; set $N_t^{(i)}$ via~\eqref{eq:systematic}
    \ELSE
      \STATE $m_t^{(i)} \leftarrow K \cdot \dfrac{w_t^{(i)} \exp(\tilde{r}_i)}{\sum_j w_t^{(j)} \exp(\tilde{r}_j)}$, \quad $\tilde w_t^{(i)} \leftarrow w_t^{(i)} / m_t^{(i)}$
      \STATE Reorder by descending $w_t^{(i)} \exp(\tilde{r}_i) / m_t^{(i)}$; draw $U \sim \mathrm{Uniform}(0,1)$; set $N_t^{(i)}$ via~\eqref{eq:systematic}
      \STATE $i^{\star} \leftarrow \arg\max_j\, r(\hat{x}_0(x_t^{(j)},t))$
      \STATE For each slot $s$: if $s$ is an extra copy of its parent, redirect: $\text{indices}[s] \leftarrow i^{\star}$
      \STATE $w_t^{(i)} \leftarrow 1$ for all $i$
    \ENDIF
    \STATE $w_t^{(i)} \leftarrow \tilde w_t^{(i)}$ for surviving offspring
  \ENDIF
  \STATE $x_{t-1}^{(i)} \leftarrow \mathrm{DDIM}_{\eta=1}(x_t^{(i)}, t)$ for all $i$
\ENDFOR
\STATE Draw $N_{\mathrm{eval}}$ samples from $\{x_0^{(i)}\}$ with probabilities $\propto w_0^{(i)} \exp(r(x_0^{(i)}))$
\RETURN Selected samples
\end{algorithmic}
\end{algorithm}

\paragraph{Variance-Decomposition Principle.}
The design of \textsc{VASR} follows from separating the two variance sources introduced by a resampling step. Let $\mathcal{F}_t$ denote the current weighted particle set, let $m_t^{(i)}$ be the expected offspring count assigned to particle $i$, and let $\tilde w_t^{(i)}=w_t^{(i)}/m_t^{(i)}$ be the reweighted offspring weight. Write $S(x_t^{(i)})$ and $\sigma_t^2(x_t^{(i)})$ for the conditional mean and variance of the terminal reward score obtained by continuing particle $i$ through the remaining denoising process.
\begin{proposition}[Stagewise variance and plug-in allocation]
\label{prop:main-stagewise}
Conditional on $\mathcal{F}_t$, the terminal estimator variance decomposes as
\[
\mathrm{Var}(\hat Z_T \mid \mathcal{F}_t)
= V_t^{\mathrm{cont}} + V_t^{\mathrm{res}},
\]
where the expected continuation term is
\begin{equation}
\mathbb{E}[V_t^{\mathrm{cont}}\mid \mathcal{F}_t]
=
\sum_{i=1}^K \frac{(w_t^{(i)})^2}{m_t^{(i)}}\,
\sigma_t^2(x_t^{(i)}).
\label{eq:main-cont-var}
\end{equation}
Here $V_t^{\mathrm{res}}$ is the variance contribution from the random offspring counts.
Using the plug-in scale proxy $\sigma_t(x_t^{(i)}) \propto e^{\tilde r_t^{(i)}}$, the minimizer of~\eqref{eq:main-cont-var} over $\{m:\sum_i m_t^{(i)}=K\}$ satisfies $m_t^{(i)} \propto w_t^{(i)} e^{\tilde r_t^{(i)}}.$
\end{proposition}
\begin{proof}
See Appendix~\ref{app:vasr-allocation}.
\end{proof}

\textsc{VASR} evolves $K$ weighted DDIM trajectories ($\eta=1$) and applies a systematic resampling step at a predefined schedule $\mathcal{T} \subset \{1,\ldots,T\}$. At each $t \in \mathcal{T}$, given cumulative masses $M_0 = 0$, $M_k = \sum_{i=1}^k m_t^{(i)}$, a single $U \sim \mathrm{Uniform}(0,1)$ determines all offspring counts via
\begin{equation}
N_t^{(k)} = \sum_{j=1}^K \mathbf{1}\{U + j - 1 \in (M_{k-1}, M_k]\}.
\label{eq:systematic}
\end{equation}

\begin{proposition}[Variance Bounds for Systematic Resampling {\cite{sys-2}}]
\label{prop:sys-var_main}
Under the systematic resampling rule~\eqref{eq:systematic},
\citet{sys-2} establish that every particle satisfies
\begin{equation}
|N_t^{(k)} - m_t^{(k)}| \le 1 \;\text{a.s.},
\qquad
\mathrm{Var}(N_t^{(k)}) \le \tfrac{1}{4},
\label{eq:systematic-bound}
\end{equation}
regardless of mass allocation. In contrast, multinomial resampling induces
variance $\mathrm{Var}(N_t^{(k)}) = m_t^{(k)}(1 - m_t^{(k)}/K) = \mathcal{O}(K)$
for heavy-mass particles~\citep{sys-2}.
\end{proposition}
\begin{proof}
The bounds follow directly from~\citet{sys-2}; we restate them here in our notation for completeness. See Appendix~\ref{app:resampling-variance} for the self-contained derivation.
\end{proof}

This represents a per-particle reduction from $\mathcal{O}(K)$ to $\mathcal{O}(1)$ that is strongest precisely when some masses are large, the regime induced by exponentially tilted rewards. Each offspring inherits weight $\tilde w_t^{(k)} = w_t^{(k)}/m_t^{(k)}$, and systematic resampling remains unbiased so this reweighting preserves the target weighted measure in expectation~(Proposition~\ref{prop:unbiased}). Together, Propositions~\ref{prop:main-stagewise} and~\ref{prop:sys-var_main} provide a complete variance control strategy: the variance-optimal allocation~\eqref{eq:var_opt_mass} minimizes $V_t^{\mathrm{cont}}$, while systematic resampling bounds $V_t^{\mathrm{res}}$. The procedure is summarised in Algorithm~\ref{alg:vasr}; its two key design choices are described below.

\paragraph{Variance-Optimal Mass Allocation.}
We score each particle using the Tweedie reward proxy $\hat{r}_t^{(i)} = r(\hat{x}_0(x_t^{(i)},t))$, optionally z-scored across particles for numerical stability when reward scales vary across steps. The masses are then set by the plug-in continuation-variance criterion in Proposition~\ref{prop:main-stagewise}:
\begin{equation}
m_t^{(i)} \;\propto\; w_t^{(i)}\,e^{\tilde{r}_i},
\qquad \sum_i m_t^{(i)} = K,
\label{eq:var_opt_mass}
\end{equation}
where $\tilde{r}_i$ denotes the standardized score. This allocates offspring in proportion to each particle's weighted reward, concentrating compute on promising trajectories while keeping the total population fixed at $K$. In high-noise regimes where the goal is to obtain a single high-quality sample, such as text-to-image generation with pretrained reward models, stronger selection can improve robustness to noisy intermediate scores. In this regime we concentrate all offspring on the highest-scoring particle,
\begin{equation}
m_t^{(i)} = K \cdot \mathbf{1}\bigl\{i = \textstyle\arg\max_j\,
w_t^{(j)} e^{\tilde{r}_j}\bigr\},
\qquad \tilde w_t^{(i)} = 1,
\label{eq:vasr_max_mass}
\end{equation}
making an explicit bias--variance trade-off in favor of stable high-reward continuation. We refer to this variant as \textsc{VASR-Max}.

\paragraph{Ordering Heuristic.}
Before drawing the systematic grid we reorder particles in decreasing order of $w_t^{(i)} e^{\tilde{r}_i} / m_t^{(i)}$. Under~\eqref{eq:var_opt_mass} this ratio is approximately constant across particles, so neighboring cells in the cumulative-mass interval tend to have similar descendant scores, further reducing the residual variance $V_t^{\mathrm{res}}$ without affecting unbiasedness \cite{sys-2}. We verify this empirically in Appendix~\ref{sec:add_expt}.

\section{Experiments}
\label{sec:experiments}

We present the main empirical results, evaluating \textsc{VASR} across two settings of increasing complexity. We first consider class-conditional posterior sampling on MNIST and CIFAR-10 using \textsc{VASR}, and then large-scale text-to-image generation with Stable Diffusion, where we use the \textsc{VASR-Max} variant (Section~\ref{sec:method}). Unless stated otherwise, results are averaged over five random seeds, with $\lambda = 1.0$ for all methods. We conduct all experiments on 4 NVIDIA GH200 GPUs.

\begingroup
\setlength{\tabcolsep}{3pt}
\renewcommand{\arraystretch}{1.2}
\newcommand{\valunc}[2]{\textnormal{#1}{\scriptsize\,(\textnormal{#2})}}
\newcommand{\highlight}[1]{\cellcolor{blue!10}{#1}}
\begin{table*}[t]
\caption{\textbf{Quantitative comparison on MNIST and CIFAR-10.} We report mean $\pm$ standard deviation over 5 seeds after $10^6$ NFEs. Metrics include FID, MMD, expected reward, and diversity (Div). Best values per column are \textbf{underlined}, and values within 5\% of the best are highlighted. \textsc{VASR} consistently achieves the best or near-best performance across all settings.} 
\label{tab:main_results}
\centering
\resizebox{\textwidth}{!}{%
\begin{tabular}{l cccc cccc cccc}
\toprule
& \multicolumn{4}{c}{\textbf{MNIST}}
& \multicolumn{4}{c}{\textbf{MNIST even/odd}}
& \multicolumn{4}{c}{\textbf{CIFAR-10}} \\
\cmidrule(lr){2-5}\cmidrule(lr){6-9}\cmidrule(lr){10-13}
Algorithm
& FID$_\downarrow$ & MMD$_\downarrow$ & $\log r_\uparrow$ & Div$_\uparrow$
& FID$_\downarrow$ & MMD$_\downarrow$ & $\log r_\uparrow$ & Div$_\uparrow$
& FID$_\downarrow$ & MMD$_\downarrow$ & $\log r_\uparrow$ & Div$_\uparrow$ \\
\midrule
DPS & \valunc{1.008}{4.173} & \valunc{0.360}{0.236} & \valunc{-0.167}{0.125} & \highlight{\valunc{0.487}{0.054}} & \valunc{1.471}{2.059} & \valunc{0.449}{0.144} & \valunc{-0.186}{0.059} & \valunc{0.491}{0.007} & \valunc{0.256}{0.060} & \valunc{1.297}{0.504} & \valunc{-0.273}{0.110} & \highlight{\valunc{0.531}{0.019}} \\
FKD & \valunc{0.037}{0.023} & \valunc{0.163}{0.130} & \valunc{-0.026}{0.013} & \valunc{0.457}{0.050} & \valunc{0.056}{0.022} & \valunc{0.241}{0.112} & \valunc{-0.032}{0.013} & \valunc{0.456}{0.007} & \valunc{0.241}{0.054} & \valunc{0.923}{0.294} & \valunc{-0.181}{0.158} & \highlight{\valunc{0.517}{0.025}} \\
TDS & \valunc{0.101}{0.057} & \valunc{0.485}{0.390} & \underline{\highlight{\valunc{-0.013}{0.030}}} & \valunc{0.429}{0.053} & \valunc{0.143}{0.066} & \valunc{0.785}{0.430} & \valunc{-0.018}{0.032} & \valunc{0.430}{0.010} & \valunc{0.418}{0.175} & \valunc{2.100}{1.221} & \underline{\highlight{\valunc{-0.043}{0.035}}} & \valunc{0.486}{0.036} \\
DAS & \valunc{0.030}{0.015} & \valunc{0.120}{0.076} & \valunc{-0.023}{0.012} & \valunc{0.456}{0.052} & \valunc{0.037}{0.010} & \valunc{0.159}{0.072} & \valunc{-0.027}{0.009} & \valunc{0.457}{0.010} & \valunc{0.205}{0.043} & \valunc{0.738}{0.199} & \valunc{-0.537}{0.236} & \highlight{\valunc{0.532}{0.021}} \\
DTS & \valunc{0.019}{0.002} & \valunc{0.100}{0.016} & \valunc{-0.020}{0.001} & \highlight{\valunc{0.494}{0.003}}
    & \highlight{\valunc{0.014}{0.004}} & \underline{\highlight{\valunc{0.075}{0.029}}} & \underline{\highlight{\valunc{-0.015}{0.004}}} & \underline{\highlight{\valunc{0.616}{0.053}}}
    & \valunc{0.180}{0.036} & \valunc{0.744}{0.183} & \valunc{-0.301}{0.098} & \highlight{\valunc{0.540}{0.019}} \\
\midrule
\textbf{VASR (ours)} & \underline{\highlight{\valunc{0.015}{0.005}}} & \underline{\highlight{\valunc{0.087}{0.043}}} & \valunc{-0.020}{0.002} & \underline{\highlight{\valunc{0.495}{0.092}}}
    & \underline{\highlight{\valunc{0.014}{0.001}}} & \valunc{0.084}{0.010} & \highlight{\valunc{-0.016}{0.001}}& \highlight{\valunc{0.616}{0.056}}
    & \underline{\highlight{\valunc{0.132}{0.026}}} & \underline{\highlight{\valunc{0.598}{0.153}}} & \valunc{-0.205}{0.099} & \underline{\highlight{\valunc{0.541}{0.019}}} \\

\bottomrule
\end{tabular}%
}
\end{table*}
\endgroup

\paragraph{Posterior Sampling under Class Conditioning.}



\begin{figure}[!htbp]
  \centering
  \vspace{-10pt}
  \includegraphics[width=\textwidth]{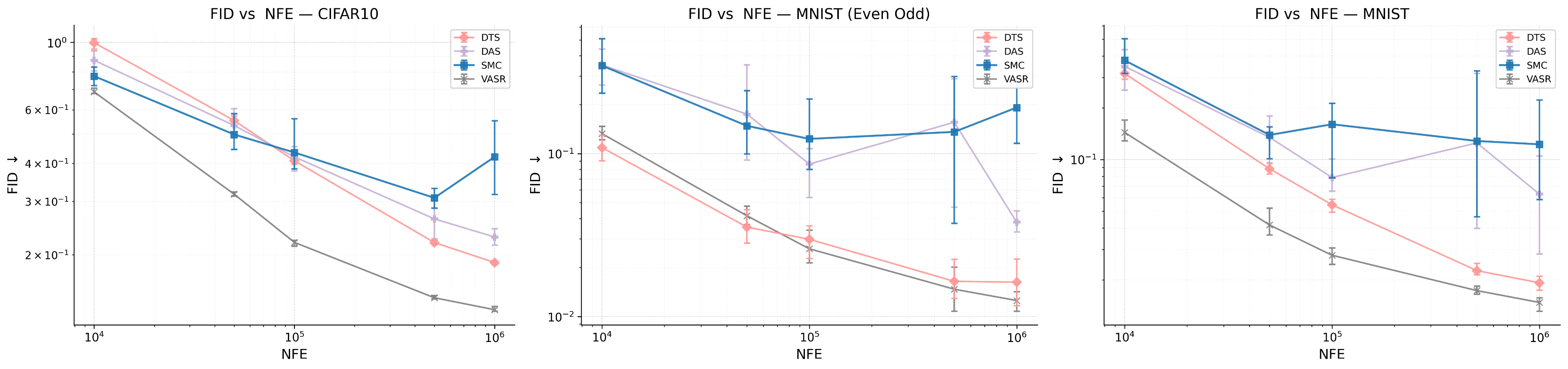}
  \vspace{-10pt}
\caption{
\textbf{Scaling behavior with increasing compute (NFEs).} FID vs NFEs on CIFAR-10 and MNIST. \textsc{VASR} outperforms FKD and DTS across all compute budgets, demonstrating favorable scaling with full parallelism.
}
  \label{fig:scaling}
\end{figure}

We evaluate class-conditional posterior sampling, where the goal is to sample from
$
p(\mathbf{x}\mid c) \propto p_{\theta}(\mathbf{x})\,p(c\mid \mathbf{x}),
$
with $p_{\theta}(\mathbf{x})$ an unconditional diffusion prior and $p(c\mid \mathbf{x})$ a pretrained classifier. Following~\cite{dts}, we consider MNIST and CIFAR-10 across all 10 classes, using the log-classifier likelihood as the reward, $r(\mathbf{x}) = \log p(c\mid\mathbf{x})$, and report results averaged over classes and seeds in Table~\ref{tab:main_results}. For MNIST, we also study a multimodal setting with even digits $\mathcal{C}_{\mathrm{even}}$ and odd digits $\mathcal{C}_{\mathrm{odd}}$, using $r(\mathbf{x}) = \log \max_{c \in \mathcal{C}} p(c\mid\mathbf{x})$, with results averaged over both groups.

We compare \textsc{VASR} against DTS~\cite{dts}, FKD~\cite{fkd}, TDS~\cite{tds}, DAS~\cite{das}, and DPS~\cite{dps}. Table~\ref{tab:main_results} reports mean and standard deviation over five seeds at $10^6$ NFEs, with metrics computed on 5000 samples. For each method, we report the best performance under reward-weighted or uniform final sampling. Figure~\ref{fig:scaling} shows FID versus NFEs to study scaling with compute; DPS and TDS are excluded due to high cross-seed variance that obscures comparison with stronger baselines.

\paragraph{Results.}
On single-class MNIST, \textsc{VASR} achieves the lowest FID and MMD while maintaining competitive reward, better diversity than FKD, and competitive diversity with DTS. This trend holds in the multimodal MNIST (even/odd) setting and extends to CIFAR-10, where \textsc{VASR} again attains lower FID and MMD with competitive reward and diversity. Although TDS achieves high rewards, Figure~\ref{fig:cifar_comparison} shows this is due to mode collapse rather than true alignment with the target posterior. Figure~\ref{fig:scaling} shows that \textsc{VASR} scales consistently with compute, outperforming FKD and DTS at higher NFE budgets while remaining approximately $66\times$ faster than DTS. This is notable since particle-based methods typically exhibit noisy scaling behavior~\cite{dts} (see Figures~\ref{fig:scaling} and~\ref{fig:sd_scaling}). We attribute this to variance-optimal mass allocation and systematic resampling, which bound offspring-count deviations by $|N_t^{(k)} - m_t^{(k)}| \le 1$ and avoid the $\mathcal{O}(K)$ variance of multinomial resampling, improving coverage of the twisted distribution.

Qualitative analysis of CIFAR-10 results (class: car) in Figure~\ref{fig:cifar_comparison} at $10^6$ NFEs shows that FKD and TDS produce near-identical samples, indicating severe collapse, while \textsc{VASR} and DTS preserve substantially higher diversity. In contrast, DPS generates samples outside the base model support, degrading quality despite high diversity, consistent with prior work~\cite{dts}.

\paragraph{Text-to-Image Generation.}
\begin{figure}[!htbp]
    \centering
    \begin{subfigure}[b]{0.49\linewidth}
        \includegraphics[width=\linewidth]{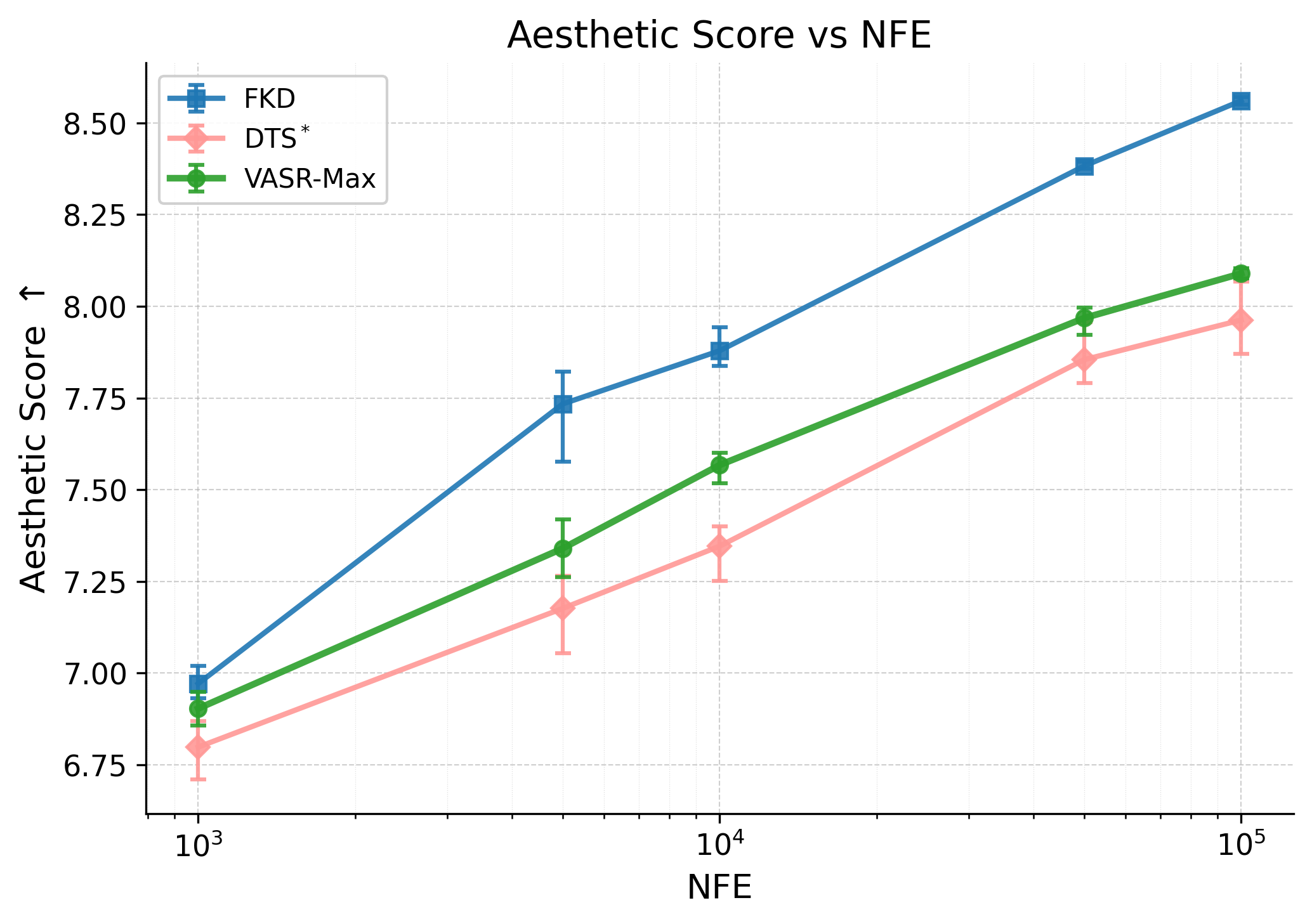}
        \label{fig:aes_scaling}
    \end{subfigure}
    \hfill
    \begin{subfigure}[b]{0.49\linewidth}
        \includegraphics[width=\linewidth]{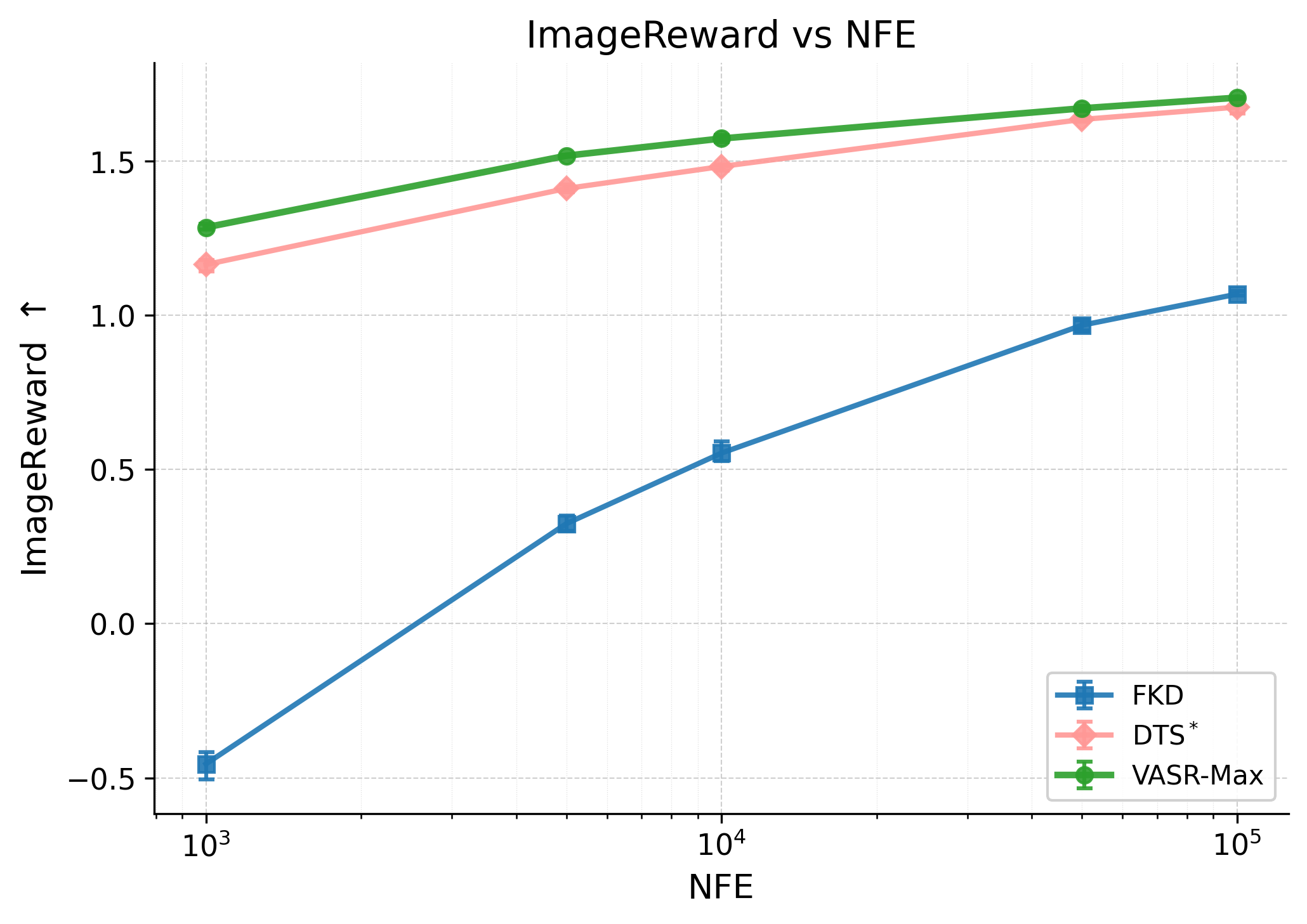}
        \label{fig:ir_scaling}
    \end{subfigure}
    \caption{
    \textbf{Reward scaling with compute in text-to-image generation.} (Left) Aesthetic score on the simple animals benchmark; (right) ImageReward on DrawBench. Methods are evaluated under matched NFE budgets. \textsc{VASR-Max} achieves competitive rewards while avoiding over-optimization and maintaining better efficiency than DTS.
}
    \label{fig:sd_scaling}
\end{figure}

In this setting, we study inference-time alignment for text-to-image generation, aiming to improve prompt adherence and perceptual quality without modifying the diffusion model. This is challenging due to the high-dimensional output space and mismatch between reward models and the data distribution. As discussed in Section~\ref{sec:method}, rewards are computed on the Tweedie proxy $\hat{x}_0(x_t, t)$ of noisy latents, and pretrained SD1.5 reward models produce substantially noisier signals than the log-classifier likelihoods used in pixel-space benchmarks. We therefore report results using \textsc{VASR-Max}.

We evaluate \textsc{VASR-Max} against FKD (strongest SMC baseline) and $\text{DTS}^*$ (value-based baseline) using Stable Diffusion v1.5~\cite{ld} as the generative prior $p_{\theta}(x \mid y)$. Following~\cite{dts}, we use two benchmarks: \textbf{(1) DrawBench}~\cite{drawbench}, evaluated with ImageReward~\cite{imagereward} capturing prompt alignment and human preference; and \textbf{(2) Aesthetic Optimization}~\cite{ddpo}, using 45 animal-category prompts and the LAION Aesthetic Predictor~\cite{aes} to measure visual quality independent of prompt correctness. All methods are evaluated under matched compute budgets with fixed NFEs per prompt.

\paragraph{Results.}
We report average reward across prompts and 3 seeds, selecting the highest-reward sample per prompt for each method, and analyze scaling in Figure~\ref{fig:sd_scaling}. On the aesthetic benchmark \cite{ddpo} (Figure~\ref{fig:sd_scaling}, left), FKD achieves the highest raw rewards but exhibits clear over-fitting, producing samples that deviate from the data distribution (see Figure~\ref{fig:aes}). In contrast, \textsc{VASR-Max} and $\text{DTS}^*$ obtain slightly lower rewards while preserving substantially better visual fidelity, with \textsc{VASR-Max} yielding a better reward scaling. On DrawBench (Figure~\ref{fig:sd_scaling}, right), \textsc{VASR-Max} consistently outperforms FKD and $\text{DTS}^*$ across all compute budgets, achieving up to $1.6\times$ higher rewards than FKD while remaining competitive with $\text{DTS}^*$ despite being approximately $4\times$ faster. Qualitative results in Appendix Figures~\ref{fig:aes} and~\ref{fig:ir} confirm these trends: FKD produces artifact-prone or over-optimized samples, whereas \textsc{VASR-Max} better preserves alignment with both the prompt and the base model distribution.

We attribute FKD's over-optimization and \textsc{VASR-Max}'s superior reward scaling to differences in how they handle noisy reward signals. Tweedie proxy rewards are inherently stochastic at high noise levels, with intermediate scores $\hat{x}_0(x_t, t)$ exhibiting substantial noise even under a perfect reward model (\S~\ref{sec:add_expt}). \textsc{VASR-Max} mitigates this by concentrating offspring on the highest-scoring particle, making selection robust to pointwise noise rather than propagating it through soft allocations, which prevents the accumulation of reward estimation errors.

\section{Discussion}
\label{sec:ablations}

\paragraph{Continuation Variance Control.}

\begin{wrapfigure}{l}{0.48\columnwidth}
  \centering
  \vspace{-10pt}
  \includegraphics[width=\linewidth]{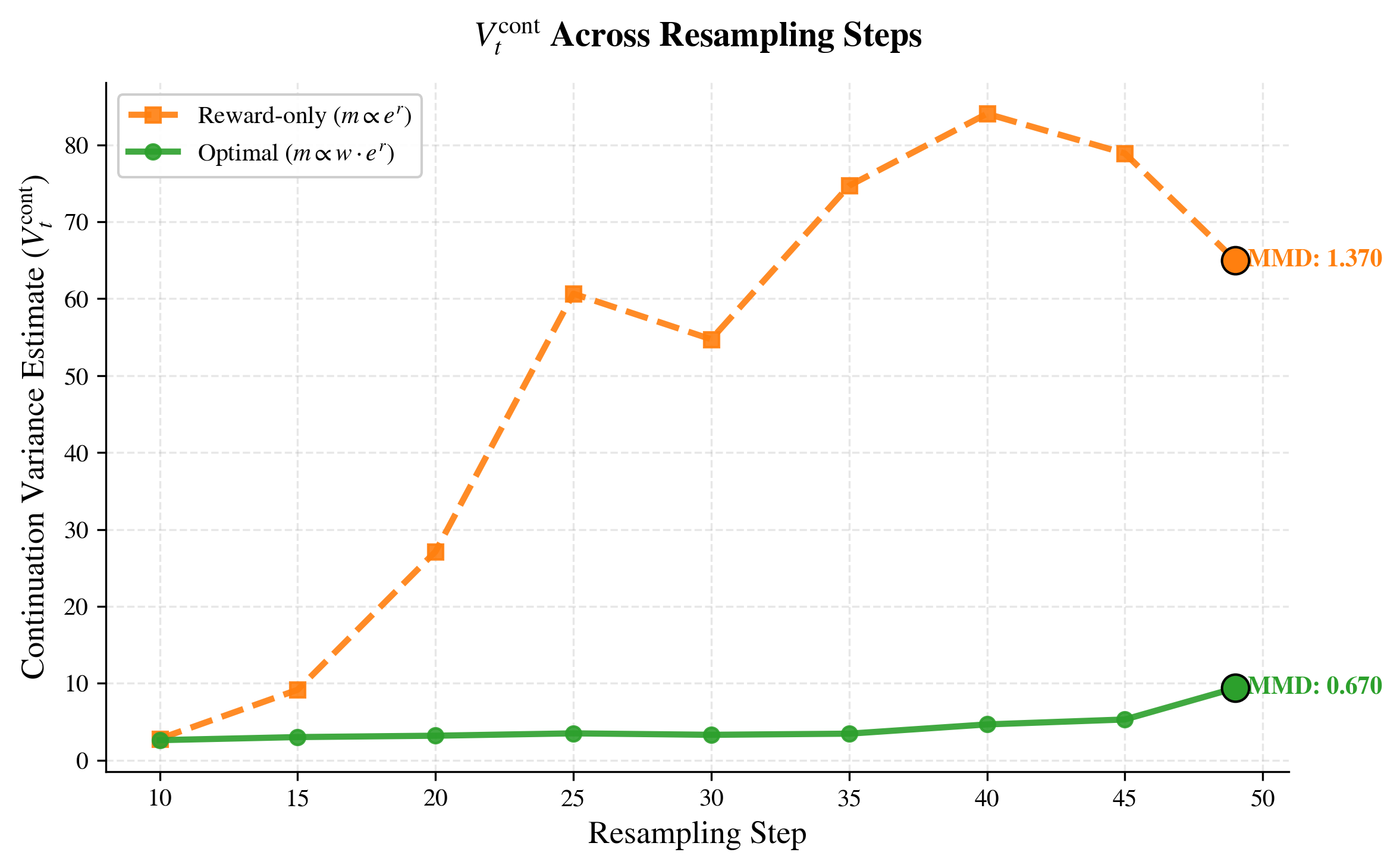}
\caption{
\textbf{$V_t^{\mathrm{cont}}$ under different mass allocations.}
The variance-optimal allocation $m \propto w \cdot e^r$ maintains substantially lower $V_t^{\mathrm{cont}}$ than the baseline $m \propto e^r$ on CIFAR-10 ($K{=}1000$), yielding nearly $2\times$ better terminal MMD.
}
  \vspace{-20pt}
    \label{fig:vcont}
\end{wrapfigure}

To empirically validate the variance-optimal mass allocation from Proposition~\ref{prop:main-stagewise}, we compare our allocation $m_t^{(k)} \propto w_t^{(k)} e^{\tilde{r}_k}$ with the baseline $m_t^{(k)} \propto e^{\tilde{r}_k}$ used in prior SMC methods~\cite{fkd}. Figure~\ref{fig:vcont} shows results on CIFAR-10 with $K{=}1000$ particles. The variance-optimal allocation reduces the continuation variance proxy by approximately $8\times$ and improves terminal MMD by nearly $2\times$, confirming that the allocation directly controls $V_t^{\mathrm{cont}}$ as predicted by the variance decomposition. We further see that lower continuation variance corresponds to better terminal sample quality, validating our prediction that controlling $V_t^{\mathrm{cont}}$ (via optimal allocation) is essential for high-quality reward-guided sampling.

\paragraph{Lineage Diversity and Sample Quality.}

\begin{wrapfigure}{r}{0.48\columnwidth}
  \centering
  \vspace{-10pt}
  \includegraphics[width=\linewidth]{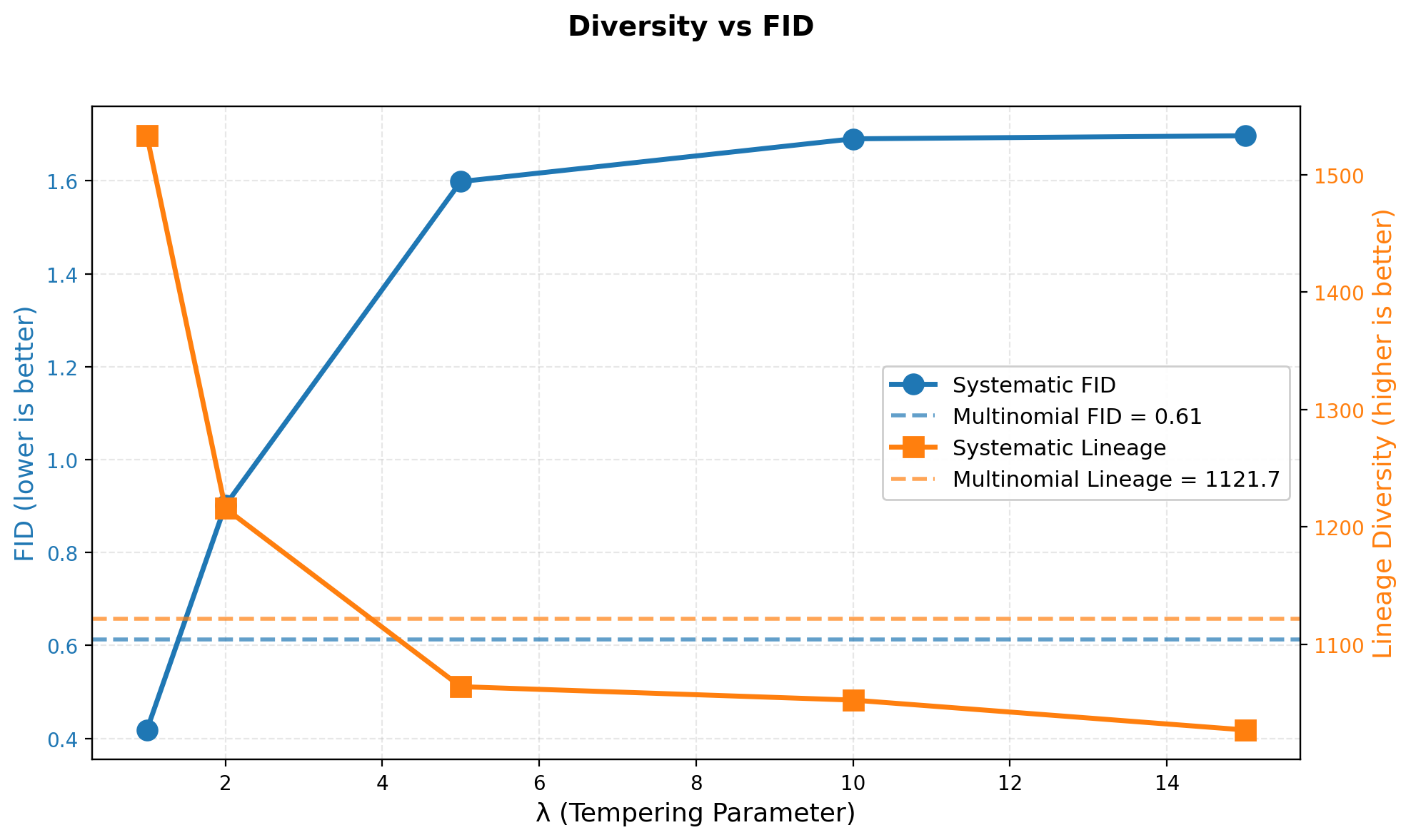}
    \caption{Temperature $\lambda$ ablation. Higher $\lambda$ reduces lineage diversity and degrades FID, approximately matching multinomial's performance at $\lambda{=}4.0$.}
  \vspace{-10pt}
    \label{fig:lambda_ablation}
\end{wrapfigure}

To validate the role of lineage diversity in sample quality, we ablate the temperature parameter $\lambda$ under systematic resampling. As shown in Figure~\ref{fig:lambda_ablation}, increasing $\lambda$ sharpens selection by concentrating mass on higher-reward particles, improving exploitation but reducing lineage diversity and degrading distributional metrics. On CIFAR-10, increasing $\lambda$ from $1.0$ to $15.0$ reduces mean lineage diversity from $\approx1500$ to $\approx1000$, while FID worsens from $\approx 0.40$ to $\approx 1.65$. Crucially, when lineage diversity under systematic resampling approaches that of multinomial resampling, FID degrades to a similar level, indicating that the performance gap is primarily driven by preserved lineage diversity. Thus, \textsc{VASR}'s advantage stems from maintaining a more diverse set of ancestral trajectories, enabling better coverage of the reward-tilted distribution.

\paragraph{Offspring-Count Variance Analysis.}

\begin{figure}[!htbp]
  \centering
  \begin{tabular}{cc}
    \includegraphics[width=0.47\linewidth]{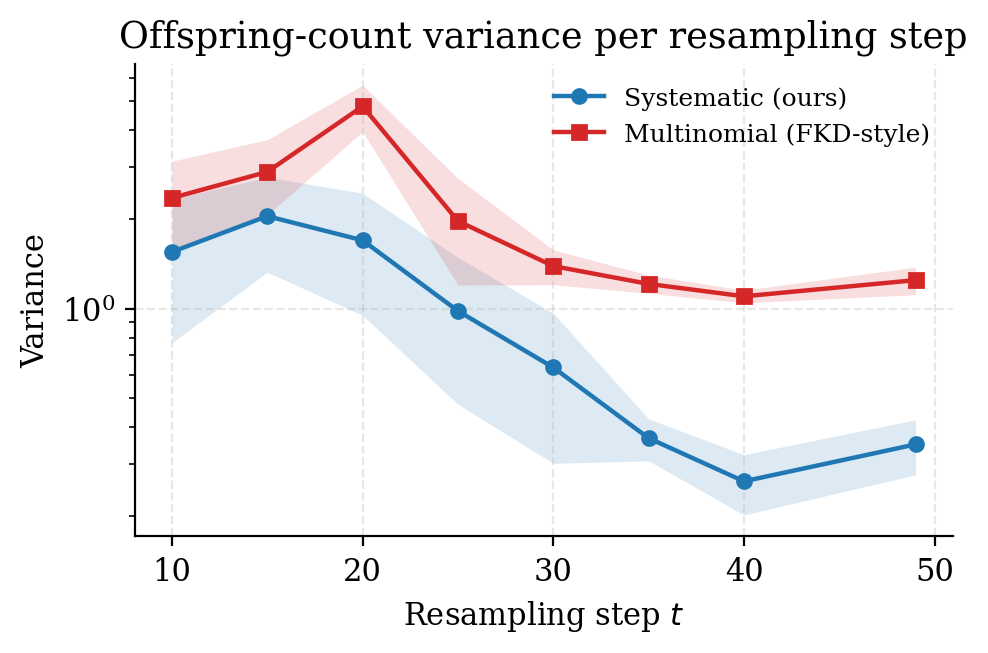} &
    \includegraphics[width=0.47\linewidth]{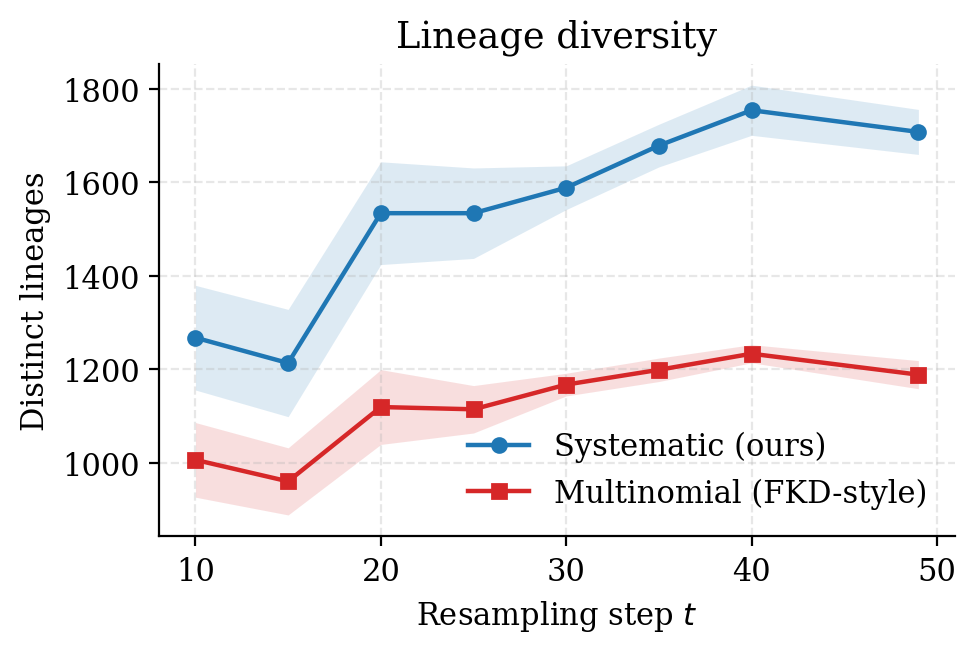} \\
  \end{tabular}
\caption{
  \textbf{Systematic vs.\ Multinomial resampling.}
  \textbf{Left:} Offspring-count variance. Systematic maintains low variance while multinomial exhibits large fluctuations.
  \textbf{Right:} Lineage diversity. Systematic retains more distinct lineages, avoiding multinomial's rapid collapse.
}
  \label{fig:death_main}
\end{figure}

A central design goal of \textsc{VASR} is controlling residual variance $V_t^{\mathrm{res}}$ from offspring sampling while maintaining selection pressure. To isolate the resampling rule's effect on $V_t^{\mathrm{res}}$, we compare systematic vs.\ multinomial resampling under identical mass allocations $m_t^{(k)} \propto w_t^{(k)} e^{\tilde{r}_k}$, where multinomial draws $(N_t^{(1)},\ldots,N_t^{(K)}) \sim \mathrm{Multinomial}(K;\,m_t^{(1)}/K,\ldots,m_t^{(K)}/K)$ and systematic uses~\eqref{eq:systematic}. Figure~\ref{fig:death_main} reports offspring-count variance $\mathrm{Var}(N_t^{(k)})$—the primary contributor to $V_t^{\mathrm{res}}$—averaged across particles, and lineage diversity (the number of distinct particles from step $t{+}1$ that produced at least one offspring at step $t$). Systematic resampling achieves lower $V_t^{\mathrm{res}}$ and retains significantly more lineages than multinomial. This confirms Proposition~\ref{prop:sys-var_main}: multinomial induces $\mathrm{Var}(N_t^{(k)}) = m_t^{(k)}(1 - m_t^{(k)}/K) = \mathcal{O}(K)$ for heavy-mass particles, whereas systematic enforces $|N_t^{(k)} - m_t^{(k)}| \le 1$ almost surely with $\mathrm{Var}(N_t^{(k)}) \le 1/4$ per particle regardless of mass \cite{sys-2}. The lineage gap is pronounced under aggressive tilting, where multinomial's large fluctuations in offspring counts increase $V_t^{\mathrm{res}}$, causing only a few particles to reproduce and collapsing diversity. All results are reported for $K{=}2000$ and $\lambda{=}1.0$, averaged across CIFAR-10 classes over 3 random seeds.

\paragraph{Inference Efficiency.}
\begin{wrapfigure}{r}{0.48\columnwidth}
  \centering
  \vspace{-10pt}
  \includegraphics[width=\linewidth]{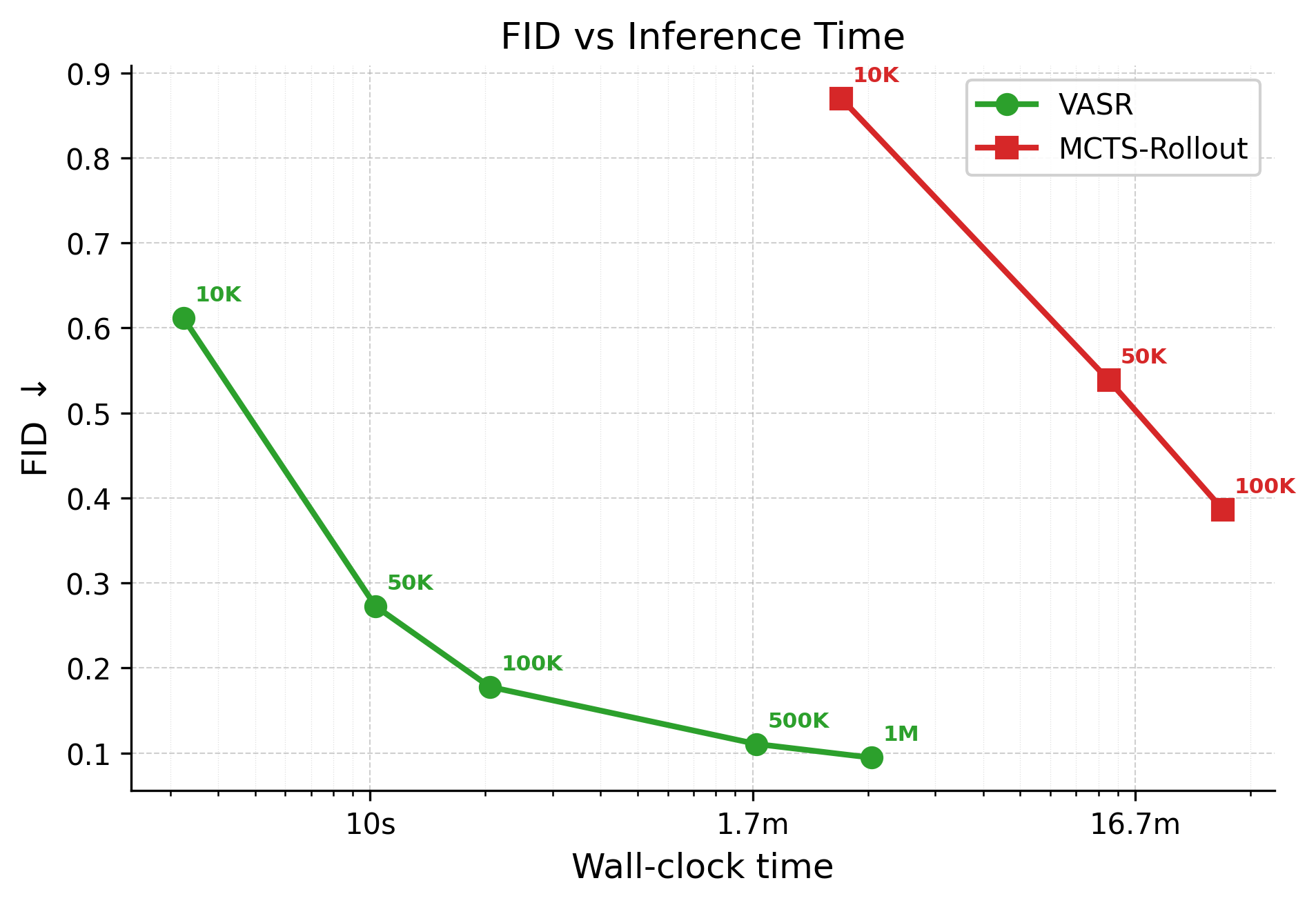}
  \vspace{-10pt}
  \caption{
\textbf{FID vs.\ time on CIFAR-10.} \textsc{VASR} achieves lower FID at all budgets and is $\sim\!66\times$ faster than DTS at matched NFEs; DTS fails to match even the worst \textsc{VASR} FID.}
  \vspace{-15pt}
  \label{fig:time_comparison}
\end{wrapfigure}
Inference-time efficiency is critical in practice, as runtime overhead can offset gains in sample quality. Value-based methods such as DTS~\cite{dts} build a search tree over the denoising trajectory, which is inherently sequential since node expansions depend on prior value estimates, limiting parallelization and increasing runtime, especially at large NFE budgets. In contrast, \textsc{VASR} evolves $K$ particles largely independently and is highly parallelizable. The systematic resampling step is $\mathcal{O}(K)$—requiring only a cumulative-mass computation and a single uniform draw—and is negligible compared to UNet evaluations, enabling efficient parallel execution.

Figure~\ref{fig:time_comparison} shows FID versus wall-clock time on CIFAR-10 (class: car). \textsc{VASR} consistently achieves lower FID than DTS across all budgets and is approximately $66\times$ faster at matched NFE. Even at its maximum runtime, DTS fails to match the best FID of \textsc{VASR}, highlighting its higher computational cost despite amortization via tree caching. This underscores the efficiency and parallelization advantages of particle-based inference.

\section{Conclusions \& Limitations}
\label{sec:conclusion}
We introduce a variance-decomposition framework for reward-guided diffusion SMC that separates estimator variance into continuation variance $V_t^{\mathrm{cont}}$ (from particle propagation) and residual variance $V_t^{\mathrm{res}}$ (from offspring sampling), motivating \textsc{VASR}. It jointly controls both via a variance-optimal mass allocation $m_t^{(k)} \propto w_t^{(k)} e^{\tilde{r}_k}$ that minimizes $V_t^{\mathrm{cont}}$, and systematic resampling that enforces $|N_t^{(k)} - m_t^{(k)}| \le 1$ to bound $V_t^{\mathrm{res}}$. Ablations show this reduces continuation variance by $\approx 8\times$, improves terminal MMD by nearly $2\times$, and preserves $\approx 1.5\times$ more ancestral lineages than multinomial resampling. Across MNIST, CIFAR-10, and Stable Diffusion, \textsc{VASR} achieves stronger distributional quality while being $\sim 4$–$66\times$ faster than tree-search methods at matched compute. In high-noise latent diffusion settings, intermediate Tweedie reward estimates are inherently noisy, motivating \textsc{VASR-Max}, which concentrates offspring on the highest-scoring particle at each step, trading unbiasedness for robustness. While effective, this sacrifices theoretical guarantees, suggesting that better denoising of intermediate rewards could enable fully unbiased variance-optimal allocation. Additional limitations include lack of evaluation on flow matching and consistency models, restriction to single-objective alignment, and absence of multi-objective and multi-modal reward analysis, which remain promising directions for future work.

\bibliographystyle{unsrtnat}
\bibliography{main}

\newpage

\appendix

\section{Additional Experiments \& Details}
 
For all VASR \& VASR-MAX experiments, we use variance-optimal mass allocation ($m_t \propto w_t \cdot e^{r_t}$) with $\lambda{=}1.0$ and the ordering heuristic enabled. Resampling parameters are task-specific: on MNIST and CIFAR-10, we resample every 5 steps from $t{=}10$ to $t{=}40$; for Stable Diffusion aesthetic score optimization (simple animal prompts), we use frequency 5 from $t{=}20$ to $t{=}45$; for ImageReward on DrawBench, we resample every 5 steps from $t{=}10$ to $t{=}40$. For SMC baselines (FKD, TDS, DAS), we found resampling frequency 10 with start and end timesteps $(10, 40)$ to work best across all settings. For MCTS-based methods (DTS), we use default parameters from the original implementation~\cite{dts}, as these yielded the strongest performance in our experiments.

\label{sec:add_expt}
\begin{figure}[!htbp]
    \centering
    \includegraphics[width=\linewidth]{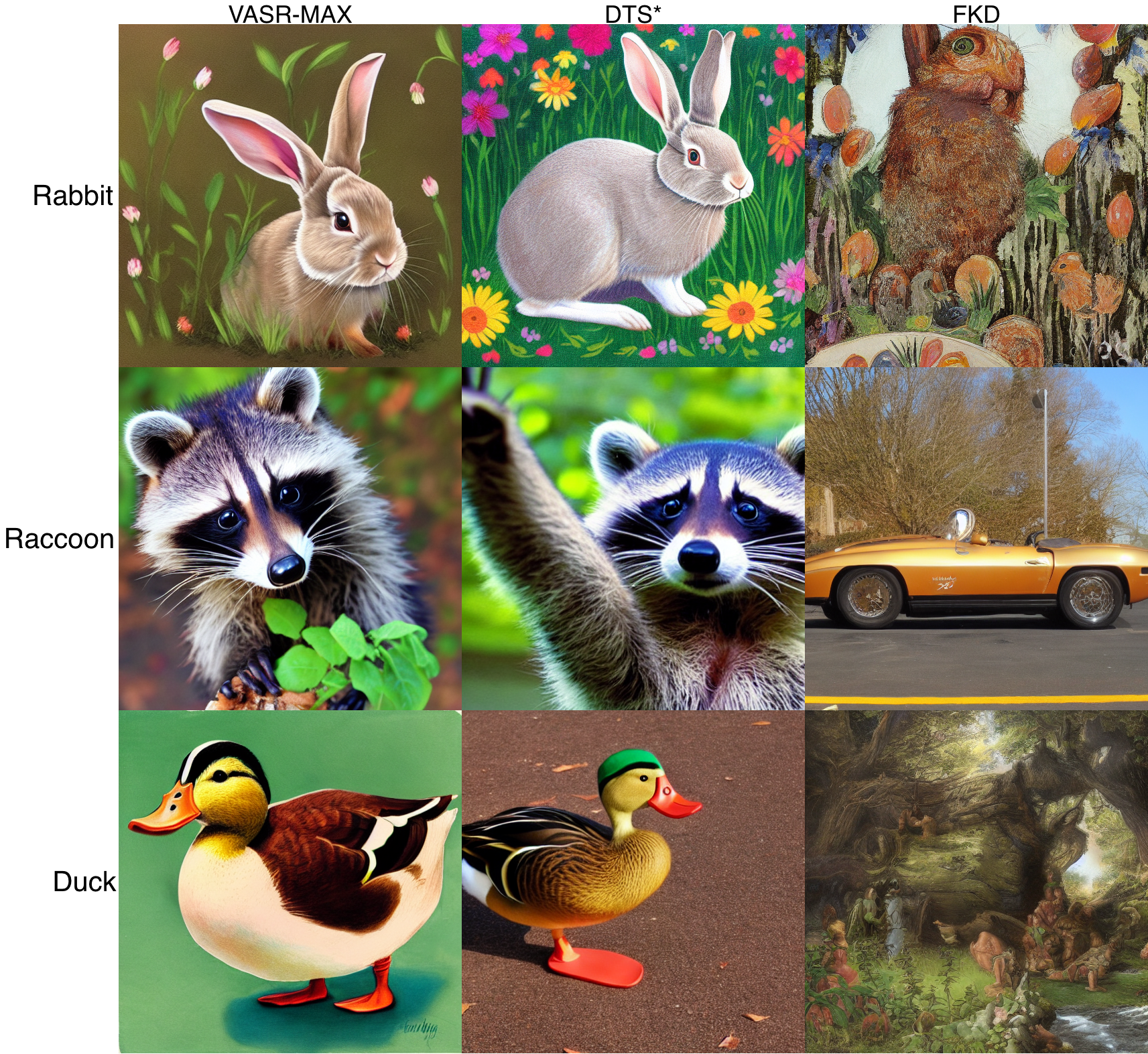}
    \caption{\textbf{Qualitative comparison on aesthetic optimization.} Samples from each method for a 100K NFE budget. As emphasized earlier, FKD achieves a high aesthetic score but produces samples that have overfit to the reward model and do not make sense \cite{seedpo}, a symptom of over-fitting. \textsc{VASR-Max} generates visually appealing images that remain faithful to the prompt.}
    \label{fig:aes}
\end{figure}

\begin{figure}[!htbp]
    \centering
    \includegraphics[width=\linewidth]{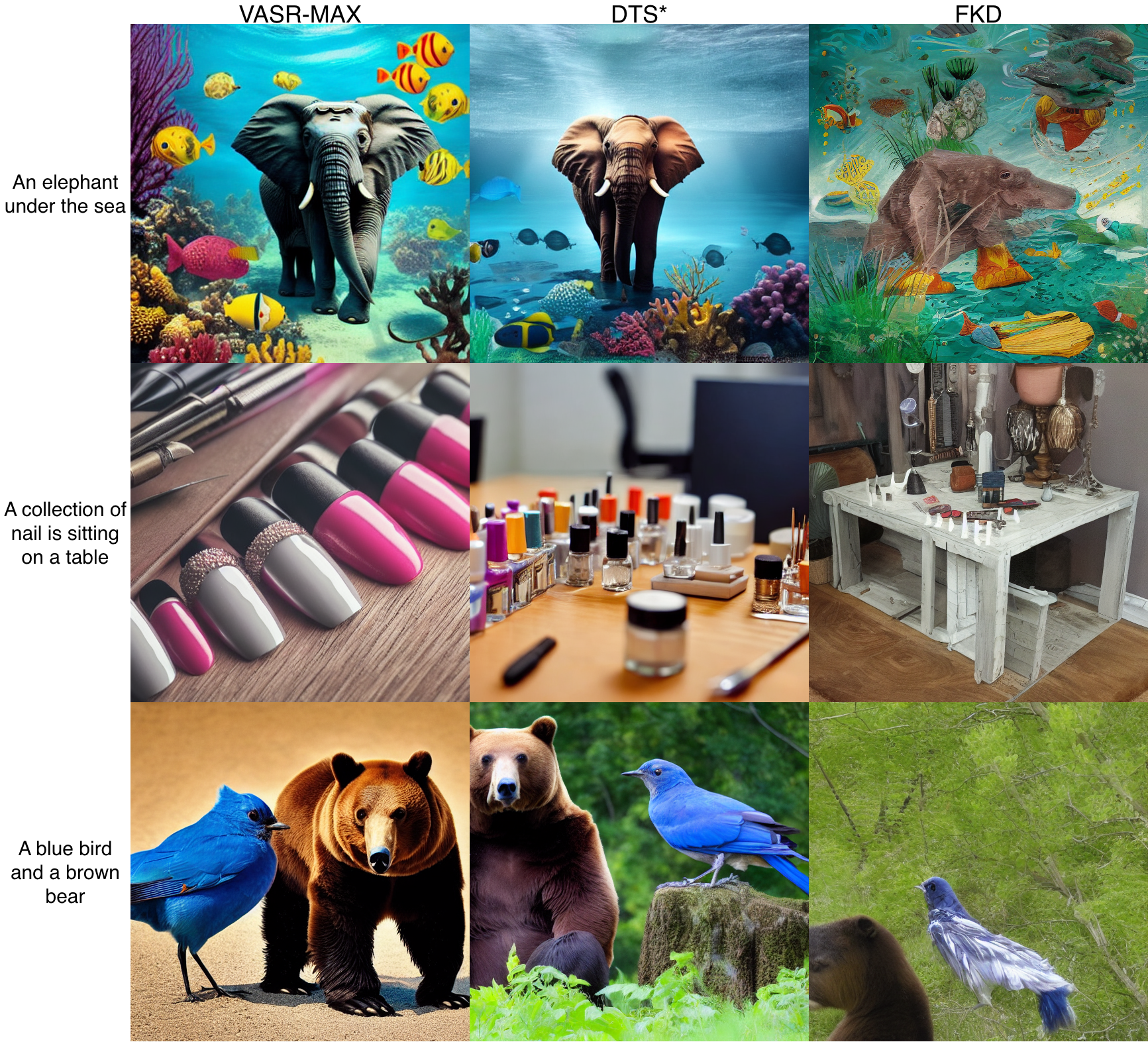}
    \caption{\textbf{Qualitative comparison on DrawBench (ImageReward).}
    Best samples generated by each method for each prompt at a 100K NFE budget. \textsc{VASR-Max} produces samples with higher prompt alignment and perceptual quality compared to DTS and FKD, consistent with the quantitative results in Figure~\ref{fig:sd_scaling}.}
    \label{fig:ir}
\end{figure}

\begin{figure}[!htbp]
  \centering
  \vspace{-10pt}
  \includegraphics[width=\textwidth]{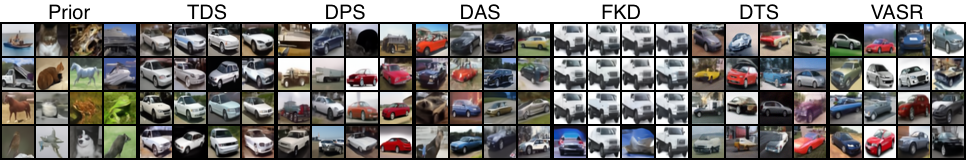}
  \vspace{-10pt}
\caption{
\textbf{Qualitative comparison on CIFAR-10 (class: car).} Samples generated with $10^6$ NFEs, selected via reward-weighted sampling. FKD and TDS exhibit mode collapse with visually similar outputs, while DPS produces out-of-distribution images. \textsc{VASR} and DTS maintain distribution alignment with higher sample diversity.
}

  \label{fig:cifar_comparison}
\end{figure}

\paragraph{Stochastic Reward Noise at Intermediate Timesteps.}
\begin{figure}[t]
    \centering
    \includegraphics[width=\linewidth]{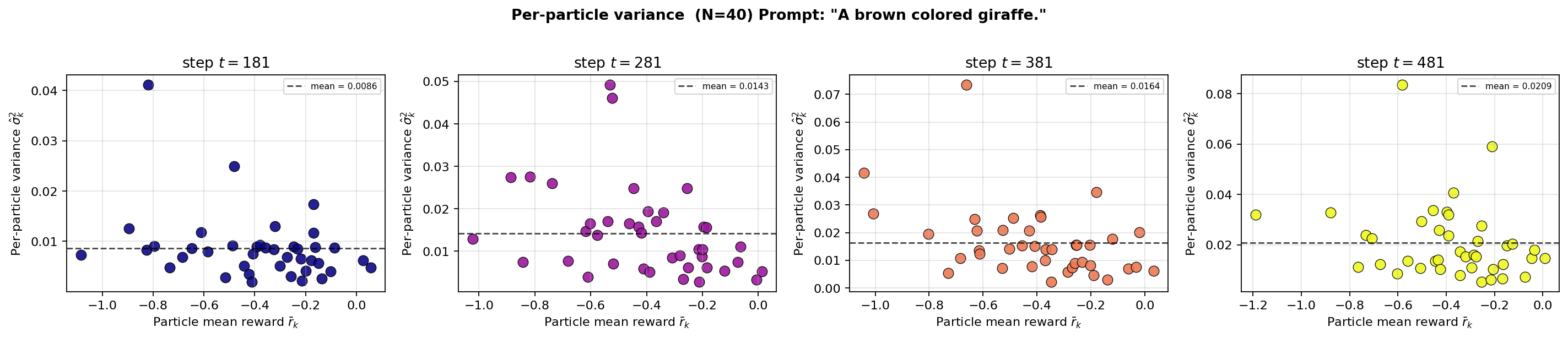}
    \caption{
\textbf{Per-particle reward variance vs.\ mean reward} (prompt: ``A brown colored giraffe.''). Each point is a particle, with the dashed line showing mean variance. Despite variation in mean rewards, within-particle variance remains nearly constant ($0.02$--$0.04$), independent of reward rank.
}
    \label{fig:reward_variance}
\end{figure}

A key assumption underlying \textsc{VASR-Max} is that intermediate reward signals are inherently noisy due to the stochasticity of the decoding process. To verify this, we run $K{=}40$ independent particles through the pipeline, and at each resampling timestep $t$ we draw $M{=}10$ stochastic rollouts ($\eta{=}1$) from each particle and score every terminal image with ImageReward. Figure~\ref{fig:reward_variance} plots the per-particle variance $\widehat{\mathrm{Var}}_k$ against mean reward $\bar{r}_k$ at each diagnostic step. Despite mean rewards spanning $[-1.0, -0.2]$, the within-particle variance remains roughly constant at $0.02$--$0.04$ with no systematic dependence on reward rank. This confirms that score differences between particles at intermediate timesteps are comparable in magnitude to the inherent rollout noise, making a soft proportional allocation unreliable. \textsc{VASR-Max} side-steps this entirely by committing to the argmax particle, whose selection is robust to small additive noise regardless of the absolute score scale.

\paragraph{Ordering Heuristic.}
Before drawing the systematic grid we reorder particles in decreasing order of $w_t^{(i)} e^{\tilde{r}_i} / m_t^{(i)}$. Under the variance-optimal mass allocation~\eqref{eq:var_opt_mass}, this ratio is approximately constant across particles, meaning that neighboring cells in the cumulative-mass interval correspond to particles with similar weighted reward scores. This spatial correlation tends to reduce fluctuations from the random grid placement beyond the $\mathcal{O}(1)$ guarantee of systematic resampling itself. Importantly, because reordering is a deterministic permutation applied before the random grid is drawn, it preserves the unbiasedness of the systematic rule.

Figure~\ref{fig:ordering} empirically confirms this variance reduction on CIFAR-10: ordering substantially reduces mean offspring-count variance at each resampling step while maintaining the bounded variance guarantee. This improved variance control translates directly to better sample quality, with terminal FID improving from $0.27$ (no ordering) to $0.24$ (with ordering). These results demonstrate that the ordering heuristic is a simple yet effective mechanism for enhancing both the statistical efficiency and distributional accuracy of variance-aware systematic resampling.

\begin{figure}[t]
    \centering
    \includegraphics[width=0.7\linewidth]{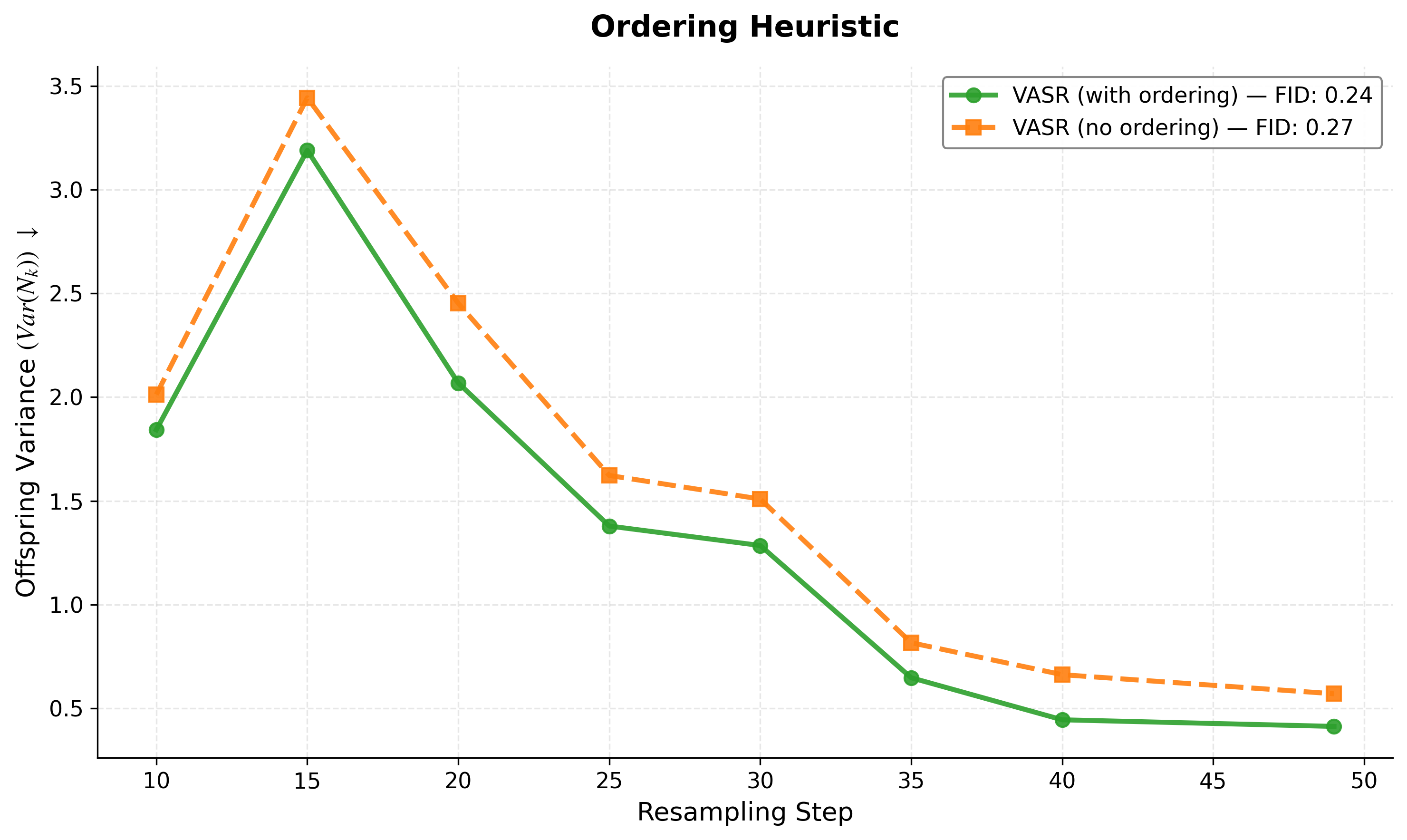}
    \caption{
        \textbf{Effect of ordering heuristic on offspring variance.}
        Ordering particles before systematic resampling reduces offspring-count variance at each step and improves terminal FID.
    }
    \label{fig:ordering}
\end{figure}

\section{Theoretical Results}
\label{app:theory}

This section collects the theoretical results referenced from the main text. Section~\ref{app:resampling-variance} treats the offspring-count variance properties of multinomial and systematic resampling (referenced from Section~\ref{sec:prelim}). Section~\ref{app:vasr-allocation} derives the variance-optimal mass allocation, the unbiasedness guarantee, and the stagewise variance decomposition that motivate the design of \textsc{VASR} (referenced from Section~\ref{sec:method}).

\subsection{Resampling Variance Bounds}
\label{app:resampling-variance}

Throughout this section we work with positive masses $m_t^{(1)},\ldots,m_t^{(K)}$ summing to $K$, and offspring counts $(N_t^{(1)},\ldots,N_t^{(K)})$ drawn under either the multinomial scheme of~\eqref{eq:multinomial} or the systematic scheme of~\eqref{eq:systematic}.

\subsubsection{Multinomial Offspring-Count Variance}
\label{app:multinomial-variance}

\begin{proposition}[Multinomial offspring-count variance~\cite{sys-2}]
\label{prop:multinomial-variance}
Let $(N_t^{(1)},\ldots,N_t^{(K)}) \sim \mathrm{Multinomial}(K;\, p_1,\ldots,p_K)$ with $p_k = m_t^{(k)}/K$ and $\sum_k m_t^{(k)} = K$. Then for every $k \in \{1,\ldots,K\}$,
\begin{equation}
\mathbb{E}[N_t^{(k)}] = m_t^{(k)},
\qquad
\mathrm{Var}(N_t^{(k)}) = m_t^{(k)}\!\left(1 - \frac{m_t^{(k)}}{K}\right).
\end{equation}
In particular, if $m_t^{(k)} = c K$ for some constant $c \in (0,1)$, then $\mathrm{Var}(N_t^{(k)}) = cK(1-c) = \Theta(K)$.
\end{proposition}

\begin{proof}
Each marginal $N_t^{(k)}$ is the number of successes in $K$ i.i.d.\ Bernoulli trials with success probability $p_k = m_t^{(k)}/K$, and is therefore distributed as $\mathrm{Binomial}(K, p_k)$. The mean and variance of a binomial are
\[
\mathbb{E}[N_t^{(k)}] = K p_k = m_t^{(k)},
\qquad
\mathrm{Var}(N_t^{(k)}) = K p_k (1 - p_k) = m_t^{(k)}\bigl(1 - m_t^{(k)}/K\bigr).
\]
The scaling claim follows by substituting $m_t^{(k)} = cK$.
\end{proof}

\subsubsection{Systematic Offspring-Count Bounds}
\label{app:systematic-bounds}

\begin{proposition}[Systematic resampling: deviation and variance~\cite{sys-2}]
\label{prop:sys-var}
Let $m_t^{(1)},\ldots,m_t^{(K)}$ be positive masses summing to $K$, with cumulative masses $M_k = \sum_{i=1}^k m_t^{(i)}$ and $M_0 = 0$. Let $U \sim \mathrm{Uniform}(0,1)$ and define the systematic offspring counts
\[
N_t^{(k)} = \sum_{j=1}^K \mathbf{1}\bigl\{U + j - 1 \in (M_{k-1}, M_k]\bigr\}, \qquad k = 1,\ldots,K.
\]
Then:
\begin{enumerate}
\item[(i)] (Unbiasedness) $\mathbb{E}[N_t^{(k)}] = m_t^{(k)}$ and $\sum_{k=1}^K N_t^{(k)} = K$ almost surely.
\item[(ii)] (Bounded deviation) $|N_t^{(k)} - m_t^{(k)}| \le 1$ almost surely.
\item[(iii)] (Variance bound) $\mathrm{Var}(N_t^{(k)}) = \{m_t^{(k)}\}\bigl(1 - \{m_t^{(k)}\}\bigr) \le \tfrac{1}{4}$, where $\{x\} = x - \lfloor x \rfloor$ denotes the fractional part.
\end{enumerate}
\end{proposition}

\begin{proof}
The bounds in (i) and (ii) follow from \citet{sys-2}; we restate them here in our notation for completeness, and extend to the explicit variance formula in (iii).

Let $S_k = (M_{k-1}, M_k]$ denote the interval of length $m_t^{(k)}$ allocated to particle $k$ on the line $(0, K]$. The point pattern $\{U + j - 1 : j = 1,\ldots,K\}$ consists of $K$ equally spaced points with unit gap, with the first point uniformly distributed in $(0,1]$.

\smallskip
\emph{(i) Unbiasedness.} Conditional on $U$, the indicator $\mathbf{1}\{U + j - 1 \in S_k\}$ is deterministic, and
\[
\mathbb{E}[N_t^{(k)}] = \int_0^1 \sum_{j=1}^K \mathbf{1}\{u + j - 1 \in S_k\}\, du
= \int_{0}^{K} \mathbf{1}\{v \in S_k\}\, dv = |S_k| = m_t^{(k)},
\]
where we changed variables $v = u+j -1$ and used the fact that the $K$ unit-length intervals $(j-1, j]$ for $j = 1,\ldots,K$ partition $(0,K]$. The almost-sure conservation $\sum_k N_t^{(k)} = K$ holds because every one of the $K$ points falls in exactly one $S_k$.

\smallskip
\emph{(ii) Bounded deviation.} Since the $K$ points are equally spaced with gap one, the number of points falling in any half-open interval $S_k$ of length $m_t^{(k)}$ is either $\lfloor m_t^{(k)}\rfloor$ or $\lceil m_t^{(k)}\rceil$:
\begin{itemize}
\item If $m_t^{(k)}$ is an integer, then exactly $m_t^{(k)}$ points fall in $S_k$ for every value of $U$, so $N_t^{(k)} = m_t^{(k)}$ almost surely.
\item If $m_t^{(k)}$ is not an integer, write $m_t^{(k)} = \lfloor m_t^{(k)}\rfloor + \{m_t^{(k)}\}$. The number of unit-length sub-intervals that fit fully inside $S_k$ is $\lfloor m_t^{(k)}\rfloor$, and the residual sub-interval of length $\{m_t^{(k)}\}$ may or may not contain a point depending on the offset $U$. So $N_t^{(k)} \in \{\lfloor m_t^{(k)}\rfloor, \lceil m_t^{(k)}\rceil\}$.
\end{itemize}
In both cases $|N_t^{(k)} - m_t^{(k)}| \le 1$ almost surely.

\smallskip
\emph{(iii) Variance bound.} If $m_t^{(k)}$ is an integer, $N_t^{(k)} = m_t^{(k)}$ almost surely and the variance is zero. Otherwise let $\phi = \{m_t^{(k)}\} \in (0,1)$. From (ii), $N_t^{(k)}$ takes only two values, $\lfloor m_t^{(k)}\rfloor$ and $\lfloor m_t^{(k)}\rfloor + 1$. By (i), $\mathbb{E}[N_t^{(k)}] = m_t^{(k)} = \lfloor m_t^{(k)}\rfloor + \phi$, so
\[
\Pr[N_t^{(k)} = \lceil m_t^{(k)}\rceil] = \phi,
\qquad
\Pr[N_t^{(k)} = \lfloor m_t^{(k)}\rfloor] = 1 - \phi.
\]
Therefore $N_t^{(k)} - \lfloor m_t^{(k)}\rfloor \sim \mathrm{Bernoulli}(\phi)$, and
\[
\mathrm{Var}(N_t^{(k)}) = \phi(1 - \phi) = \{m_t^{(k)}\}\bigl(1 - \{m_t^{(k)}\}\bigr) \le \tfrac{1}{4},
\]
with equality at $\phi = 1/2$.
\end{proof}

\begin{remark}
The variance bound in Proposition~\ref{prop:sys-var}(iii) is independent of $K$ and of $m_t^{(k)}$ itself --- it depends only on the fractional part $\{m_t^{(k)}\}$. Combined with the multinomial variance in Proposition~\ref{prop:multinomial-variance}, this gives the per-particle improvement from $\mathcal{O}(K)$ to $\mathcal{O}(1)$ stated in the main text~\cite{sys-2}.
\end{remark}

\subsection{\textsc{VASR} Mass Allocation and Variance Decomposition}
\label{app:vasr-allocation}

This section derives the unbiasedness guarantee, the variance-optimal mass allocation, and the stagewise variance decomposition that justify the design of \textsc{VASR}. Throughout we write $\tilde{r}_t^{(i)} := r(\hat{x}_0(x_t^{(i)},t))$ for the standardized Tweedie reward proxy evaluated at particle $i$, and $\mathcal{F}_t$ for the sigma-algebra generated by the particle history up to time $t$.

\subsubsection{Unbiasedness of the Reweighted Empirical Measure}
\label{app:unbiased-proof}

\begin{proposition}[\textsc{VASR} preserves the weighted empirical measure]
\label{prop:unbiased}
Let the particle set $\{(x_t^{(i)}, w_t^{(i)})\}_{i=1}^K$ be propagated through one resampling step at time $t \in \mathcal{T}$ with masses $m_t^{(i)}$ and offspring counts $N_t^{(i)}$ satisfying $\mathbb{E}[N_t^{(i)}] = m_t^{(i)}$ and $\sum_i N_t^{(i)} = K$ almost surely. If each offspring of particle $i$ is assigned weight $\tilde w_t^{(i)} = w_t^{(i)}/m_t^{(i)}$, then for every bounded measurable test function $f$ on $\mathcal{X}$,
\[
\mathbb{E}\!\left[\sum_{i=1}^K N_t^{(i)} \tilde w_t^{(i)} f(x_t^{(i)}) \,\middle|\, \mathcal{F}_t\right]
= \sum_{i=1}^K w_t^{(i)} f(x_t^{(i)}).
\]
\end{proposition}

\begin{proof}
Conditional on $\mathcal{F}_t$, the particles $x_t^{(i)}$ and masses $m_t^{(i)}$ are deterministic. By linearity of expectation and the unbiasedness condition $\mathbb{E}[N_t^{(i)}\mid\mathcal{F}_t]=m_t^{(i)}$,
\[
\mathbb{E}\!\left[\sum_{i=1}^K N_t^{(i)} \tilde w_t^{(i)} f(x_t^{(i)}) \,\middle|\, \mathcal{F}_t\right]
= \sum_{i=1}^K \frac{w_t^{(i)}}{m_t^{(i)}} f(x_t^{(i)})\, m_t^{(i)}
= \sum_{i=1}^K w_t^{(i)} f(x_t^{(i)}). \qedhere
\]
\end{proof}

\begin{remark}
Proposition~\ref{prop:unbiased} holds for any unbiased resampling scheme. It is the foundation of \textsc{VASR}: the choice of resampling rule affects only the variance of the estimator, not its expectation. \textsc{VASR-Max} uses the spike allocation~\eqref{eq:vasr_max_mass}, which assigns zero mass to all but one particle and therefore violates the positivity assumption $m_t^{(i)} > 0$; the unbiasedness guarantee does not apply, and \textsc{VASR-Max} is a deliberately biased estimator chosen for robustness to noisy reward signals.
\end{remark}

\subsubsection{Variance-Optimal Mass Allocation}
\label{app:opt-alloc-proof}

At each $t \in \mathcal{T}$, \textsc{VASR} sets masses as
\begin{equation}
m_t^{(i)} = K \cdot \frac{w_t^{(i)}\,e^{\tilde{r}_t^{(i)}}}{\sum_j w_t^{(j)}\,e^{\tilde{r}_t^{(j)}}},
\label{eq:vasr-mass-app}
\end{equation}
where $\tilde{r}_t^{(i)}$ is the optionally standardized Tweedie reward proxy. We derive this allocation from a plug-in continuation-variance criterion in which the proxy captures the scale of each particle's future terminal score.

Let $S(x_t^{(i)}) := \mathbb{E}[e^{r(X_T)} \mid x_t^{(i)}]$ denote the conditional terminal score of particle $i$, and $\sigma_t^2(x_t^{(i)}) := \mathrm{Var}(e^{r(X_T)} \mid x_t^{(i)})$ its conditional variance under stochastic DDIM propagation ($\eta{=}1$). After resampling, each of the $N_t^{(i)}$ offspring of particle $i$ is propagated independently with weight $\tilde w_t^{(i)} = w_t^{(i)}/m_t^{(i)}$. Taking the expectation of the continuation variance with respect to the unbiased offspring counts gives
\begin{equation}
\mathbb{E}[V_t^{\mathrm{cont}} \mid \mathcal{F}_t]
= \sum_{i=1}^K \frac{(w_t^{(i)})^2}{m_t^{(i)}}\,\sigma_t^2(x_t^{(i)}).
\label{eq:Vt-cont-expected}
\end{equation}

\begin{proposition}[Variance-optimal mass allocation]
\label{prop:opt-alloc}
Let $a_i>0$ denote a proxy for the continuation standard deviation $\sigma_t(x_t^{(i)})$, and suppose $a_i \propto e^{\tilde r_t^{(i)}}$. Then the minimizer of the surrogate
\[
\sum_{i=1}^K \frac{(w_t^{(i)})^2 a_i^2}{m_t^{(i)}}
\]
over the simplex $\{m \in \mathbb{R}_{>0}^K : \sum_i m_t^{(i)} = K\}$ is
\[
m_t^{(i),\star} \propto w_t^{(i)}\,e^{\tilde{r}_t^{(i)}},
\]
which is exactly the allocation used by \textsc{VASR}~\eqref{eq:var_opt_mass}.
\end{proposition}

\begin{proof}
Define $b_i := w_t^{(i)} a_i > 0$ and minimize $\sum_i b_i^2 / m_i$ subject to $\sum_i m_i = K$ via Lagrange multipliers:
\[
\mathcal{L}(m,\mu) = \sum_i \frac{b_i^2}{m_i} + \mu\!\left(\sum_i m_i - K\right).
\]
The stationarity condition $\partial\mathcal{L}/\partial m_i = 0$ gives
\[
-\frac{b_i^2}{m_i^2}+\mu=0,
\qquad
m_i^\star = \frac{b_i}{\sqrt{\mu}}.
\]
Thus $m_i^\star \propto b_i = w_t^{(i)}a_i$. If $a_i \propto e^{\tilde r_t^{(i)}}$, then $m_i^\star \propto w_t^{(i)}e^{\tilde r_t^{(i)}}$. The objective is strictly convex in $m$ on the positive orthant, so this stationary point is the unique global minimum on the simplex.
\end{proof}

\begin{remark}
The reward-weighted allocation~\eqref{eq:var_opt_mass} uses $e^{\tilde r_t^{(i)}}$ as a tractable plug-in estimate of future terminal-score scale. If the absolute continuation variance is locally constant, $\sigma_t^2(x_t^{(i)}) \approx \tau_t^2$, the exact minimizer of~\eqref{eq:Vt-cont-expected} reduces to the weight-only allocation $m_i^\star \propto w_t^{(i)}$. In high-noise regimes where the plug-in score scale is less reliable, \textsc{VASR-Max} provides the complementary robust selection rule used in our text-to-image experiments.
\end{remark}

\subsubsection{Stagewise Variance Decomposition}
\label{app:stagewise-proof}

\begin{proposition}[Stagewise variance decomposition]
\label{prop:stagewise-var}
The conditional variance of the terminal estimator at step $t$ admits the decomposition
\[
\mathrm{Var}(\hat Z_T \mid \mathcal{F}_t)
= V_t^{\mathrm{cont}} + V_t^{\mathrm{res}},
\]
where $V_t^{\mathrm{cont}}$ is defined in~\eqref{eq:Vt-cont-expected} and the resampling term is the full quadratic form
\[
V_t^{\mathrm{res}}
= \sum_{i,j=1}^K
\tilde w_t^{(i)}\tilde w_t^{(j)}
S(x_t^{(i)})S(x_t^{(j)})
\mathrm{Cov}(N_t^{(i)},N_t^{(j)}\mid\mathcal{F}_t).
\]
Its diagonal contribution satisfies
\begin{equation}
D_t^{\mathrm{res}}
:= \sum_{i=1}^K
(\tilde w_t^{(i)})^2 S(x_t^{(i)})^2
\mathrm{Var}(N_t^{(i)}\mid\mathcal{F}_t)
\le \frac{1}{4}\sum_{i=1}^K
(\tilde w_t^{(i)})^2 S(x_t^{(i)})^2
\label{eq:Vt-res-diag-bound}
\end{equation}
under systematic resampling.
\end{proposition}

\begin{proof}
Apply the law of total variance to $\hat Z_T = \sum_i N_t^{(i)} \tilde w_t^{(i)} S_T^{(i)}$ with respect to the offspring counts $N$:
\[
\mathrm{Var}(\hat Z_T \mid \mathcal{F}_t)
= \underbrace{\mathbb{E}[\mathrm{Var}(\hat Z_T \mid \mathcal{F}_t, N)\mid\mathcal{F}_t]}_{V_t^{\mathrm{cont}}}
+ \underbrace{\mathrm{Var}(\mathbb{E}[\hat Z_T\mid\mathcal{F}_t,N]\mid\mathcal{F}_t)}_{V_t^{\mathrm{res}}}.
\]
For the first term, conditional on $N$ and $\mathcal{F}_t$, the descendants of distinct particles propagate independently under DDIM ($\eta{=}1$), so $\mathrm{Var}(\hat Z_T \mid \mathcal{F}_t, N) = \sum_i N_t^{(i)}(\tilde w_t^{(i)})^2 \sigma_t^2(x_t^{(i)})$. Taking expectation over $N$ with $\mathbb{E}[N_t^{(i)}]=m_t^{(i)}$ recovers~\eqref{eq:Vt-cont-expected}. For the second term, $\mathbb{E}[\hat Z_T \mid \mathcal{F}_t, N] = \sum_i N_t^{(i)} \tilde w_t^{(i)} S(x_t^{(i)})$, so taking the variance over $N$ gives the covariance quadratic form in the statement. The diagonal bound~\eqref{eq:Vt-res-diag-bound} follows by retaining the $i=j$ terms and applying $\mathrm{Var}(N_t^{(i)}\mid\mathcal{F}_t) \le 1/4$ from Proposition~\ref{prop:sys-var}.
\end{proof}

\begin{remark}
Proposition~\ref{prop:stagewise-var} separates per-step variance into a continuation term $V_t^{\mathrm{cont}}$ and a resampling term $V_t^{\mathrm{res}}$. The allocation \eqref{eq:vasr-mass-app} targets the continuation term through an importance-weighted score-scale surrogate, while systematic resampling controls the offspring-count variances that appear in the diagonal part of $V_t^{\mathrm{res}}$~\cite{sys-2}. The ordering heuristic is not needed for unbiasedness; it is used to reduce the empirical effect of the off-diagonal covariance terms by placing particles with similar weighted continuation scores next to each other on the systematic grid.
\end{remark}

\newpage

\newpage

\end{document}